\documentclass{article}

\usepackage[a4paper, total={6in, 8in}]{geometry}

\usepackage{amsmath,amssymb,amsthm,natbib,graphicx,url}

\usepackage[utf8]{inputenc} 
\usepackage[T1]{fontenc}    
\usepackage{booktabs}       
\usepackage{amsfonts}       
\usepackage{nicefrac}       
\usepackage{microtype}      
\usepackage{enumitem}
\usepackage{xspace,color,colortbl}
\usepackage{mathrsfs, mathtools}
\usepackage{authblk}

\usepackage{algorithm}
\usepackage[noend]{algpseudocode}

\usepackage{thmtools} 
\usepackage{thm-restate}

\usepackage{pgfplots}
\pgfplotsset{compat=1.14}
\usepackage{tikz}
    \usetikzlibrary{shapes,decorations}
    \usetikzlibrary{positioning}
    \usetikzlibrary{cd}
\tikzcdset{every label/.append style = {font = \normalsize}}
\usepackage{adjustbox}

\usepackage{caption}
\usepackage[symbol]{footmisc}
\usepackage[pdfencoding=auto,hidelinks]{hyperref}

\usepackage[capitalize]{cleveref} 
\renewcommand{\tilde}{\widetilde}

\newcommand{\E}[1]{\mathbb{E}\left[#1\right]}
\newcommand{\eps}{\varepsilon}
\renewcommand{\P}[1]{\mathbb{P}\left(#1\right)}

\renewcommand{\U}{\mathcal{U}}

\newcommand{\ind}[1]{\mathbb{I}{\left\{#1\right\}}}

\newcommand{\N}{\mathbb{N}}

\newcommand{\cA}{\mathcal{A}}
\newcommand{\cB}{\mathcal{B}}

\newcommand{\cE}{\mathcal{E}}
\newcommand{\cF}{\mathcal{F}}

\newcommand{\cS}{\mathcal{S}}

\newcommand{\cU}{\mathcal{U}}

\newcommand{\cX}{\mathcal{X}}
\newcommand{\cY}{\mathcal{Y}}

\newcommand{\bbE}{\mathbb{E}}
\newcommand{\I}{\mathbb{I}}
\newcommand{\bbL}{\mathbb{L}}
\newcommand{\Pb}{\mathbb{P}}
\newcommand{\bbQ}{\mathbb{Q}}

\newcommand{\e}{\varepsilon}
\newcommand{\fhi}{\varphi}

\newcommand{\lrb}[1]{\left(#1\right)}
\newcommand{\brb}[1]{\bigl(#1\bigr)}
\newcommand{\Brb}[1]{\Bigl(#1\Bigr)}
\newcommand{\bbrb}[1]{\biggl(#1\biggr)}
\newcommand{\lsb}[1]{\left[#1\right]}
\newcommand{\bsb}[1]{\bigl[#1\bigr]}
\newcommand{\Bsb}[1]{\Bigl[#1\Bigr]}

\newcommand{\bcb}[1]{\bigl\{#1\bigr\}}
\newcommand{\Bcb}[1]{\Bigl\{#1\Bigr\}}

\newcommand{\lce}[1]{\left\lceil#1\right\rceil}

\newcommand{\lfl}[1]{\left\lfloor#1\right\rfloor}

\newcommand{\lno}[1]{\left\lVert#1\right\rVert}
\newcommand{\bno}[1]{\bigl\lVert#1\bigr\rVert}

\newcommand{\diff}{\,\mathrm{d}}
\newcommand{\dif}{\mathrm{d}}

\newcommand{\m}{\setminus}
\newcommand{\iop}{\infty}

\newcommand{\perturb}{\Xi}
\newcommand{\leb}{\mathbb{L}}
\newcommand{\cprob}{c_{\mathrm{prob}}}
\newcommand{\cplat}{c_{\mathrm{plat}}}
\newcommand{\cspike}{c_{\mathrm{spike}}}
\newcommand{\kl}{\mathcal{D}_{\mathrm{KL}}}
\definecolor{mygreen}{RGB}{155,205,155}
\definecolor{myblue}{RGB}{150,210,240}
\definecolor{myred}{RGB}{255,150,150}
\newcommand{\blindExpThree}{Blind-Exp3}
\newcommand{\fullfed}{Price-Hedge}
\newcommand{\shortfullfed}{\text{Price-Hedge}}

\newcommand{\gener}{multi-apple tasting}

\newcommand{\gft}{\textsc{GFT}}
\newcommand{\egft}{\widehat \gft}

\newtheorem{claim}{Claim}

\newtheorem{lemma}{Lemma}

\newtheorem{definition}{Definition}
\newtheorem{theorem}{Theorem}

\newtheorem{observation}{Observation}

\title{Repeated Bilateral Trade\\ Against a Smoothed Adversary\thanks{Partially supported by ERC Advanced Grant 788893 AMDROMA ``Algorithmic and Mechanism Design Research in Online Markets'' and MIUR PRIN project ALGADIMAR ``Algorithms, Games, and Digital Markets''. 
The work of TC has benefited from the AI Interdisciplinary Institute ANITI, which is funded by the French ``Investing for the Future --- PIA3'' program under the Grant agreement ANR-19-P3IA-0004. 
TC also acknowledges the support of the project BOLD from the French national research agency (ANR), and that of IBM.
NCB and RC also acknowledge the support of the EU Horizon 2020 ICT-48 research and innovation action under grant agreement 951847, project ELISE (European Learning and Intelligent Systems Excellence) and the FAIR (Future Artificial Intelligence Research) project, funded by the NextGenerationEU program within the PNRR-PE-AI scheme.
} }
\usepackage{times}

\author[1,5]{Nicol\`o Cesa-Bianchi}
\author[2]{Tommaso Cesari}
\author[1,3]{Roberto Colomboni}
\author[4]{\\ Federico Fusco}
\author[4]{Stefano Leonardi}

\affil[1]{Universit\`a degli Studi di Milano, Milano, Italy}
\affil[2]{University of Ottawa, Ottawa, Canada}
\affil[3]{Istituto Italiano di Tecnologia, Genova, Italy}
\affil[4]{Sapienza Universit\`a di Roma, Roma, Italy}
\affil[5]{Politecnico di Milano, Milano, Italy}

\begin{document}

\maketitle

\begin{abstract}
We study repeated bilateral trade where an adaptive $\sigma$-smooth adversary generates the valuations of sellers and buyers. 
We provide a complete characterization of the regret regimes for fixed-price mechanisms under different feedback models in the two cases where the learner can post either the same or different prices to buyers and sellers.
We begin by showing that the minimax regret after $T$ rounds is of order $\sqrt{T}$ in the full-feedback scenario. 
Under partial feedback, any algorithm that has to post the same price to buyers and sellers suffers worst-case linear regret. 
However, when the learner can post two different prices at each round, we design an algorithm enjoying regret of order $T^{3/4}$ ignoring log factors.
We prove that this rate is optimal by presenting a surprising $T^{3/4}$ lower bound, which is the main technical contribution of the paper.
\end{abstract}

\clearpage

\section{Introduction}
In the bilateral trade problem, two strategic agents---a seller and a buyer---wish to trade some good. They both privately hold a personal valuation for it and strive to maximize their respective quasi-linear utility. The solution to the problem consists in designing a mechanism that intermediates between the two parties to make the trade happen. 
In general, an ideal mechanism for the bilateral trade problem would optimize the efficiency, i.e., the gain in social welfare resulting from trading the item from seller to buyer, while enforcing incentive compatibility (IC) and individual rationality (IR). 
The assumption that makes a two-sided mechanism design more complex than its one-sided counterpart is budget balance (BB): the mechanism cannot subsidize the market. 
Unfortunately, as \citet{Vickrey61} observed in his seminal work, the optimal incentive compatible mechanism maximizing social welfare for bilateral trade may not be budget balanced.
A more general result due to \citet{MyersonS83} shows that there are some problem instances where a fully efficient mechanism for bilateral trade that satisfies IC, IR, and BB does not exist.  This impossibility result holds even if prior information on the buyer and seller's valuations is available and the truthful notion is relaxed to Bayesian incentive compatibility.
To circumvent this obstacle, the subsequent vast body of work primarily considers the Bayesian version of the problem, where agents' valuations are drawn from some distribution and the efficiency is evaluated in expectation with respect to the valuations' randomness.
There are many incentive compatible mechanisms that give a constant approximation to the social welfare---see, e.g., \citet{BlumrosenD14, Duetting20}, and more recently to the harder problem of approximating the gain from trade \citep{DengMSW21}. 
Although in some sense necessary---without any information on the priors there is no way to extract any meaningful approximation to the social welfare \citep{Duetting20}---the Bayesian assumption of perfect knowledge of the valuations' underlying distributions is unrealistic. 

Following recent work \citep{Nicolo21,AzarFF22}, we study this fundamental mechanism design problem in an online learning setting where at each time $t$, a new seller/buyer pair arrives.
The seller has a private valuation $s_t \in [0,1]$ representing the smallest price they are willing to accept in order to trade. 
Similarly, the buyer has a private value $b_t \in [0,1]$ representing the highest price they will pay for the item.
We assume both valuations are generated by an adversary.
Independently, the learner posts two (possibly randomized) prices: $p_t \in [0,1]$ to the seller and $q_t\in [0,1]$ to the buyer. We require budget balance: it must hold that $p_t \le q_t$ for all $t$ or, equivalently, that the pair $(p_t,q_t)$ belongs to the upper triangle $\cU = \bcb{ (x,y) \in [0,1]^2 \mid x \le y }$.
A trade happens if and only if both agents agree to trade, i.e., when $s_t \le p_t$ and $ q_t \le b_t$.
When this is the case, the learner observes some feedback $z_t$ and is awarded the gain from trade at time $t$:
\[
	\gft_t(p, q) = \brb{ (b_t - q) + (p - s_t) } \cdot\ind{s_t \le p \le q \le b_t}\footnote{Other works considered the similar definition $(b_t-s_t)\cdot\ind{s_t \le p \le q \le b_t}$. All our results translate with minimal effort to this definition as well.}.
\]
When the two prices $p$ and $q$ are equal, we omit one of the arguments to simplify the notation. When we want to stress the dependence on the valuations, we use the notation $\gft(p,q,s_t,b_t)$ instead of $\gft_t(p,q)$. 
We refer the reader to the Learning Protocol of Sequential Bilateral Trade Against a $\sigma$-smooth Adversary (the definition of $\sigma$-smoothness is recalled below).
The regret of a learning algorithm $\cA$ against an adversary $\cS$ generating the sequence of random pairs $(S_t,B_t)$ is defined by:
\[
        R_T(\cA,\cS) = \max_{(p,q)\in\U}\E{\sum_{t=1}^T \gft_t(p,q) - \sum_{t=1}^T \gft_t(P_t,Q_t)}.
\]
We use $P_t, Q_t$ to stress that the prices are possibly randomized, with the convention that uppercase letters refer to random variables and the corresponding lowercase letters to their realizations.
The expectation in the previous formula is then with respect to the internal randomization of the learning algorithm and the adversary. 
The regret $R_T(\cA)$ of a learning algorithm $\cA$ is defined as its performance against the hardest adversary, i.e., as the supremum over all adversaries $\cS$ (in a certain class we define in the next paragraph) of $R_T(\cA,\cS)$. 
Our goal is to study the minimax regret $R_T^\star$, which measures the performance of the best algorithm against the worst possible adversary, i.e., the infimum over all algorithms $\cA$ of $R_T(\cA)$.
The set of learning algorithms we allow varies with the different settings we consider, i.e., with how many prices are posted and what feedback is available---see below. 
\begin{algorithm*}[t]
\caption*{\textbf{Learning Protocol of Sequential Bilateral Trade Against a $\sigma$-smooth Adversary}}
\begin{algorithmic}[t]
    \For{time $t=1,2,\ldots$}
        \State The adversary privately chooses the $\sigma$-smooth distribution of an r.v.\ $(S_t,B_t)$ on $[0,1]^2$
        \State Seller and buyer valuations $(s_t,b_t)$ are drawn from $(S_t,B_t)$
        \State The learner posts prices $(p_t,q_t) \in \cU$
        \State The learner receives a (hidden) reward $ \gft_t(p_t, q_t)\in [0,1]$
        \State Feedback $z_t$ is revealed to the learner
    \EndFor
\label{a:learning-model}
\end{algorithmic}
\end{algorithm*}

Smoothed analysis of algorithms, originally introduced by \citet{spielman2004smoothed} and later formalized for online learning by \citet{rakhlin2011online} and \citet{haghtalab2020smoothed}, is an approach to the analysis of algorithms in which the instances at every round are generated from a distribution that is not too concentrated. 

In this work, we consider a smoothed valuation-generating model that interpolates between the adversarial and the stochastic regimes. This is a natural choice for the bilateral trade problem, where algorithms with sublinear regret only exist for the stochastic i.i.d.\ setting (with additional assumptions), and where the adversarial model is known to be intractable \citep{Nicolo21}. 
At each time step $t$, a pair of valuations $(s_t,b_t)$ is sampled according to the random variable $(S_t,B_t)$, whose distribution is chosen by the adversary.
Our adversary is adaptive because the distribution of $(S_t, B_t)$ may depend on the past realizations of the valuations and the past internal randomization of the algorithm.
We focus on $\sigma$-smoothed adversaries, where the distributions of $(S_t, B_t)$ are not too concentrated, according to the following notion.
\begin{definition}[\cite{HaghtalabRS21}]
Let $X$ be a domain that supports a uniform distribution $\nu$. A measure $\mu$ on $X$ is said to be $\sigma$-smooth if for all measurable subsets $A \subseteq X$, we have $\mu(A) \le \frac{\nu(A)}{\sigma}$.
\end{definition}

We say that a random variable is $\sigma$-smooth if its distribution is $\sigma$-smooth. We consider two families of learning algorithms, corresponding to two ways of being budget balanced:
\begin{itemize}[topsep=4pt,itemsep=0pt,leftmargin=9pt]
\item{\textbf{Single-price mechanisms}.} If we want to enforce a stricter notion of budget balance, namely strong budget balance, the mechanism is neither allowed to subsidize nor extract revenue from the system. This is modeled by imposing $p_t = q_t$, for all $t$. 
\item{\textbf{Two-price mechanisms}.} If we require that the mechanism enforces (weak) budget balance, it can post two different prices, $p_t$ to the seller and $q_t$ to the buyer, as long as $p_t\le q_t$ at each time step. Namely, we only require that the mechanism never subsidize a trade; i.e., it can still make a profit.
\end{itemize}

\begin{observation} \label{obs:twopriceobs}
The only reason for a budget-balanced algorithm to post two different prices is to obtain more information.
A direct verification shows that the expected gain from trade can always be maximized by posting the same price to both the seller and the buyer. 
\end{observation}
We consider three natural types of feedback models, in increasing order of difficulty for the learner.
The last two are partial feedback models that enjoy the desirable property of requiring only a minimal amount of information from the agents:
\begin{itemize}[topsep=4pt,itemsep=0pt,leftmargin=9pt]
\item{\textbf{Full feedback.}} $z_t = (s_t,b_t)$: The learner observes both seller and buyer valuations. This model corresponds to a direct revelation mechanism. (By Observation \ref{obs:twopriceobs}, in this model, there is no reason to post two distinct prices, as all the relevant information is revealed anyway.)
\item{\textbf{Two-bit feedback.}} $z_t = \brb{ \I\{s_t \le p_t\}, \,\I\{q_t \le b_t\} }$: The learner observes separately if the two agents accept the prices offered to each of them.
\item{\textbf{One-bit feedback.}} $z_t = \I\{s_t \le p_t \le q_t \le b_t\}$: The learner only observes whether or not the trade occurs. 
This is arguably the minimal feedback the learner could get. 
\end{itemize}

\subsection{Overview of results}
We characterize (up to logarithmic terms) the dependence in the time horizon of the minimax regret regimes for the online learning version of the bilateral trade problem against an adaptive $\sigma$-smooth adversary for various feedback models and notions of budget balance, as outlined in \Cref{table:results}. We prove the following results:
    {
    \renewcommand{\arraystretch}{1.4}
    \begin{table}[tbp]
    \centering
    \begin{tabular}{c|c|c|c|}
        \cline{2-4}
        & Full Feedback & Two-bit Feedback & One-bit Feedback                        \\ \hline
        \multicolumn{1}{|l|}{Single Price} & $\tilde{O}\brb{\sqrt{T}}$ \hspace*{3mm} Theorem \ref{thm:full-upper} & $\Omega(T)$ \hspace*{3mm}& $\Omega(T)$
        \\ \hline
        \multicolumn{1}{|l|}{Two Prices}   & $\Omega \brb{ \sqrt{T} }$ & $\Omega(T^{3/4})$ \hspace*{3mm} Theorem \ref{thm:two-bit-two-prices-lower} & $\tilde{O}\brb{  T^{3/4} }$ \hspace*{3mm} Theorem \ref{thm:one-bit-two-prices-upper} 
        \\ \hline
    \end{tabular}
    \caption{\footnotesize{Overview of the regret regimes against a $\sigma$-smooth adversary. The lower bound for the full feedback model is from \citet[Thm.~3.3]{Nicolo21}, that for single price with two-bit feedback is from Thm. 5 of the same paper.}}
    \label{table:results}
    \end{table}
    }
\begin{itemize}[topsep=4pt,itemsep=0pt,leftmargin=9pt]
\item For the full feedback model, we design the \fullfed{} algorithm, that posts a single price at each time step and enjoys an $O(\sqrt{T \ln T})$ bound on the regret (\Cref{thm:full-upper}). 
By \citet[Theorem~3.3]{Nicolo21}, this rate is optimal, up to logarithmic terms.
\item For both partial-feedback models, we design the \blindExpThree{} algorithm, that posts two prices at each time step and enjoys an $\tilde O(T^{3/4})$ bound on the regret (\Cref{thm:one-bit-two-prices-upper}).
We prove that, surprisingly, this rate is optimal (\Cref{thm:two-bit-two-prices-lower}), up to logarithmic terms.
\item We prove that no algorithm can achieve worst-case sublinear regret when the platform is allowed to post a single price but receives partial feedback (one or two bits), even in the case of an i.i.d.\ $\sigma$-smooth adversary where the valuations of buyers and sellers are independent (\Cref{thm:two-bit-one-price-lower}).\footnote{%
This strengthens a result in \citet[Theorem 5]{Nicolo21}, where the same lower bound was proven without the additional restriction of independence of seller and buyer.%
}
\end{itemize}
We highlight three salient qualitative features of our results. 
First, we construct a (surprising) lower bound of order $T^{3/4}$ for the minimax regret of the problem with partial feedback where the learner is allowed to post two prices. This lower bound, which is also our main technical contribution, is strictly worse that the $T^{2/3}$ rate that can be obtained with access to bandit feedback,\footnote{Although our decision space is two-dimensional, one can see that, in a bandit feedback with a smooth adversary,
a regret of order $T^{2/3}$ can be obtained by running an optimal bandit algorithm (e.g., MOSS \citealt{audibert2009minimax}, that has an upper bound on the regret of order $\sqrt{KT}$) on a discretization of $K = \Theta(T^{1/3})$ equispaced prices on the diagonal $\{(p,q) \in \cU \mid p = q \}$. Similar results appeared, e.g., in \citet{kleinberg2004nearly,auer2007improved}.%
} 
and substantially
departs from the rates $\sqrt{T}, T^{2/3}, T$ that can be found in the two closest partial feedback models in the literature: online learning with feedback graphs \citep{AlonCGMMS17} and partial monitoring \citep{BartokFPRS14}.
Second, we introduce the first sublinear-regret learning algorithm for the partial feedback version of the bilateral trade problem beyond the (strict) stochastic i.i.d.\ assumption on the valuations.
Finally, our results imply that, from the online learning perspective, there is no difference between receiving one or two bits of feedback when two prices can be posted; this is in stark contrast, for instance, with what happens in the stochastic case when only one price can be posted---see \citet[Section~8]{Nicolo21}.
    
\subsection{Technical challenges and our techniques}

        The repeated bilateral trade problem is characterized by two key features that set it apart from the standard online learning with full or bandit feedback models: the nature of the action space and the partial feedback structure. Both these challenges need to be taken into account to construct the $T^{3/4}$ lower bound, which is the main technical endeavor of this work. 
        
        \paragraph{The action space.} The action space of the bilateral trade problem is continuous (the prices live in a subset of $[0,1]^2$), while the gain from trade is discontinuous. 
        This entails that, without any smoothness assumptions on the distributions, the problem     turns out to be utterly intractable in the standard adversarial setting (see the ``needle in a haystack'' phenomenon in Theorem 6 of \citet{Nicolo21} and Theorem 3 of \citet[]{AzarFF22}). We leverage the $\sigma$-smoothness by showing that it induces regularity on the expected gain from trade (Lemma \ref{l:lip}). This in turn allows us to prove a key discretization result (\Cref{claim:discretization}).

        \paragraph{The feedback structure.} The main peculiarity of the bilateral trade problem lies in the partial feedback models naturally associated with it. 
        Receiving only information about the relative ordering of the prices posted and the realized valuations does not allow the learner to directly reconstruct the gain from trade received at each time step. 
        For instance, if the learner posts the same price $0.5$ to both agents and they both accept, there is no way of assessing whether its gain from trade is constant (e.g., $(s,b)=(0,1)$) or arbitrarily small (e.g., $s=0.5-\e$ and $b=0.5 + \e$). Conversely, if one of the two agents rejects the price posted, the learner can only infer loose bounds on the lost trade opportunity. The key technical tool to address this challenge is given by a one-bit estimation technique that exploits the possibility of posting {\em two} prices to estimate the gain from trade it would have achieved by posting {\em one} single price to both agents \citep{Nicolo21, AzarFF22}. This tool, together with our discretization result (Lemma \ref{claim:discretization}) are behind our \blindExpThree{} algorithm achieving a $T^{3/4}$ regret.

        \paragraph{Our $T^{3/4}$ lower bound.} 
        At a (very) high level, we show that bilateral trade with partial feedback contains instances that are closely related to instances of online learning with feedback graphs \citep{AlonCDK15}.
        The corresponding feedback graph $G_K$ is over $2K$ actions: $K$ of them are ``exploring'' and the others are ``exploiting'', see \Cref{fig:feedback-graph}, left.
        Exploring actions are costly and reveals feedback on corresponding exploiting actions.
        One of the exploiting actions is optimal, but none of them returns any feedback.
        We then build ``hard'' instances so that any algorithm is forced to spend a long time playing each of the many exploring actions.
        By selecting optimally the number of arms in the reduction and the difference in reward between exploiting actions, we obtain the $T^{3/4}$ rate.
        This proof sketch hides many technical challenges: 
        we need to carefully design $\sigma$-smooth distributions of the adversary that we can map into instances of online learning with feedback graphs that achieve their lower bound.
        This presents two problems: on the one hand, the gain from trade achievable from different prices is inherently dependent (while in usual lower bound constructions for online learning with feedback graphs, the rewards of arms can be chosen arbitrarily, \citealt{AlonCDK15}), on the other hand, the embedding needs to preserve the feedback structure, which is significantly different from classic ones (such as bandits or experts) and as such, requires novel and subtle arguments. To address the second challenge, we prove a general information-theoretic result (\Cref{t:inverse-transformation-method-2}, in \Cref{appe:inverse-transform-section}) that may be of independent interest for further lower-bounds constructions in related problems.

        \subsection{Additional related works}
        \label{sec:additional}
            Recent works on the smoothed analysis of online learning algorithms include \citet{haghtalab2020smoothed}, \citet{haghtalaboracle}, and \citet{block2022smoothed}---see \Cref{sec:additional} for additional related works. 
            Further applications of smoothed analysis to online learning problems include the works by \citet{block2022efficient} and \citet{block2023smoothed}. \citet{sachs2022between} study a related stochastic adversary in the more general online convex optimization setting; however, they do not insist on the smoothness of the distributions. 
            
            In online learning settings with partial feedback, like the one we study here, smoothed analysis has been primarily applied to linear contextual bandits \citep{kannan2018smoothed,raghavan2020greedy,sivakumar2020structured,sivakumar2022smoothed}, where contexts are drawn from smooth distributions. However, the focus of those works has been on improving regret bounds specifically for the greedy algorithm that has linear regret in the worst case. Although the smoothed adversary causes the expected gain from trade to be Lipschitz, the best possible regret rates for the partial feedback models considered here are provably worse than those achievable with bandit feedback. 
            To the best of our knowledge, bilateral trade with a smoothed adversary was previously studied only by \citet{Nicolo21} in the two-bit feedback model.
            Another line of work considers regret bounds parameterized by variations of losses across time and other related measures of smoothness \citep{hazan2010extracting,chiang2012online,steinhardt2014adaptivity}. See also \citet{chen2021impossible} for recent results in this area.
            
            The minimax regret of online learning with partial feedback is rather well understood when the learner selects actions from a finite set---see, e.g., the vast literature on feedback graphs and the recent work by \citet{lattimore2022minimax} on partial monitoring. General analyses of settings with infinitely many actions sets are mostly limited to bandit feedback \citep{kleinberg2019bandits}.

\section{Warm-up: one-price setting}

In this section, we present our discretization error result (sharpening by a constant the bound in \citealt{Nicolo21}) and present our results in the single-price setting.

\subsection{Regret due to discretization.}
\label{sec:discretization}

Our first theoretical result concerns the study of how discretization impacts the regret against $\sigma$-smooth adversaries. 
Although the gain from trade is, in general, discontinuous, its expectation is $1/\sigma$-Lipschitz, thus opening the way to discretization methods, as formalized by the following result.

\begin{lemma}[Lipschitzness]
    \label{l:lip}
        Let $(S,B)$ be a $\sigma$-smooth random variable on $[0,1]^2$, then the induced expected gain from trade $\gft$ is $1/\sigma$-Lipschitz:
        \begin{equation}\label{eq:lips}
            |\E{ \gft(y) - \gft(x)}| \le \frac{1}{\sigma} |y-x|, \quad \forall x,y \in [0,1]
        \end{equation}
    \end{lemma}
    \begin{proof}
        Let $x>y$ be any two prices in $[0,1]$, we have the following:
        \begin{align*}
            |\E{ \gft(y) - \gft(x)}| &= |\E{(B-S)(\ind{S\le y \le B} - \ind{S \le x \le B})}|  \\
            &= |\E{(B-S) (\ind{S\le y \le B \le x} - \ind{y \le S \le x \le B}) }|\\
            &\le \P{S\le y \le B \le x} + \P{y \le S \le x \le B}\\
            &= \P{(S,B) \in [0,y]\times[y,x]} + \P{(S,B) \in [y,x]\times[x,1]}\\
            &\le \frac{1}{\sigma}\P{(U,V) \in [0,y]\times[y,x]} + \frac{1}{\sigma}\P{(U,V) \in [y,x]\times[x,1]} \\
            &= \frac{1}{\sigma}\left[ y \cdot (x-y) + (1-x)(x-y) \right] \le \frac{1}{\sigma} (x-y)
        \end{align*}
        Note that in the second to last inequality we used the assumption on the smoothness of $(S,B)$ and we introduced $U$ and $V$, two independent uniform random variables in $[0,1]$.
    \end{proof}

\begin{claim}[Discretization error]
\label{claim:discretization}
Let $G$ be any finite grid of prices in $[0,1]$ and let $\delta(G)$ be the largest distance of a point in $[0,1]$ to $G$, {\sl i.e.}, $\delta(G) = \max_{p \in [0,1]}\min_{g \in G} |p-g|$,
then for any sequence of $\sigma$-smooth distributions $\S = (S_1,B_1), \ldots, (S_T,B_T)$, we have the following:
\[
    \max_{p \in [0,1]}\bbE\lsb{\sum_{t=1}^T \gft_t(p)} - \max_{g \in G}\bbE\lsb{\sum_{t=1}^T \gft_t(g)}\le \frac{\delta(G)}{\sigma} T \;.
\]
\end{claim}

\begin{proof}
        Let $p^*$ be the best fixed price in hindsight in $[0,1]$ with respect to the sequence $\mathcal S$; if $p^* \in Q$, then there is nothing to prove. If this is not the case, then there exist $p_G \in G$, is such that $|p^*-p_G| \le \delta(G).$
        We have the following:
        \begin{align*}
            &\E{\sum_{t=1}^T \gft_t(p^*)} - \max_{p \in G}\E{\sum_{t=1}^T \gft_t(p)}\\ 
            & \quad \le \E{\sum_{t=1}^T \gft_t(p^*)} - \E{\sum_{t=1}^T \gft_t(p_G)} \\
            & \quad \le \frac{|p^*-p_Q|}{\sigma} \le \frac{\delta(G)}{\sigma},
        \end{align*}
        where, in the second to last inequality, we used the Lipschitz property of the expected gain from trade as in Lemma \ref{l:lip}.
    \end{proof}

\subsection{Posting a single price in full information.}
\label{sec:fullfeedback}

In the full feedback model, the learner observes a realization $z_t = (s_t,b_t)$ of $(S_t,B_t)$ at the end of each round $t$.
Thus, they are able to reconstruct the gain from trade of any other pair of prices. 
By \Cref{claim:discretization}, we can therefore run our favorite learning algorithm for (non-oblivious adversarial) online learning with expert advice on a discrete set of prices.
For example, using Hedge \citep{freund1997decision} we obtain the Price-Hedge algorithm, whose regret is controlled by the following theorem.

\begin{algorithm*}[t]
\caption*{\textbf{Learning algorithm with full feedback} : \fullfed{} }
\begin{algorithmic}[t]
\State \textbf{Input:} time horizon $T$, Hedge algorithm $\mathcal A$, grid of prices $G$, with $|G|=K$
\State \textbf{Initialization:} Initialize $\mathcal A$ on time horizon $T$ with $K$ actions, one for each $p \in G$
\For{time $t=1,2,\ldots$}
    \State Receive from $\mathcal A$ the price $p_t \sim P_t \in G$
    \State Post price $p_t$ to the agents and receive feedback $z_t = (s_t,b_t)$
    \State Feed to $\mathcal A$ the rewards $\gft_t(p) = (b_t-s_t) \ind{s_t \le p \le b_t}$, for all $p \in G$
\EndFor
\label{alg:full-feedback}
\end{algorithmic}
\end{algorithm*}

\begin{theorem}
\label{thm:full-upper}
    Consider the problem of repeated bilateral trade against a $\sigma$-smooth adversary in the full feedback model, for any $\sigma \in (0,1]$.
    Then the regret of \fullfed{}, run using the uniform $K$-grid $G$ on $[0,1]$, for $K \ge 2$, satisfies:
    \[
        R_T(\mathrm{\shortfullfed}) \le 2\sqrt{T \ln K} + \frac{T}{\sigma K}.
    \]
    In particular, for $K = \big\lfloor{\sqrt{T}}\big\rfloor$, the bound becomes: $
        R_T\mathrm{(\shortfullfed}) \le \frac{4}{\sigma} \cdot \sqrt{T \ln T}$.
\end{theorem}

 \begin{proof}
        We show how running \fullfed{} for the right choices grid of prices $G$ yields the desired result. As $\mathcal A$, we choose the Hedge algorithm for full information feedback \citep{freund1997decision}, while for any fixed $K\ge 2$, we consider the uniform grid $G$ on $[0,1]$ of the positive integer multiples of $\frac 1K$: $G=\left\{\frac{1}K, \frac 2K, \dots, 1\right\}$. For any $\sigma$-smooth adversary $\cS$, we have the following:
        \begin{align*}
            R_T(\shortfullfed, \cS) &= \max_{p \in [0,1]}\E{\sum_{t=1}^T \gft_t(p) - \sum_{t=1}^T \gft_t(P_t)} \pm \max_{p \in G}\E{\sum_{t=1}^T \gft_t(p)} \\
            &\le \max_{p \in G}\E{\sum_{t=1}^T \gft_t(p) - \sum_{t=1}^T \gft_t(P_t)} + \frac{T}{\sigma K}\\
            &\le 2 \sqrt{T \ln K} + \frac{T}{\sigma K}.
        \end{align*}
        Note that, in the first inequality we used \Cref{claim:discretization}, which holds for any (possibly adaptive) sequence of $\sigma$-smooth random variables, while in the second inequality, we used the well-known bound on the regret of Hedge---see, e.g., \citet[Theorem~2.5]{AroraHK12} with $\eta = \frac{1}{\sqrt T}$. 
        We remark that the last bound holds for any (possibly adversarial) realizations of the agents' valuations, in expectation with respect to the internal randomness of the algorithm.
    \end{proof}

\noindent We note here that the upper bound we achieved in \Cref{thm:full-upper} is tight in the time horizon, up to logarithmic terms. 
This follows from the fact that the distribution used in the $\Omega(\sqrt{T})$ lower bound in \cite[Theorem 3.3]{Nicolo21} is $\sigma$-smooth, for $\sigma \le \nicefrac{1}{4}$.

\subsection{Posting a single price in partial information.}
\label{sec:one-price}

\cite{Nicolo21} proved that sublinear regret is achievable with one price and partial information in the stochastic i.i.d.\ case, when seller and buyer distributions are smooth and independent of each other. 
They also showed that removing either the smoothness assumption or the independence of $S$ and $B$ leads to linear lower bounds. 
They did not, however, investigate whether the i.i.d.\ assumption could be lifted in a setting other than the classic adversarial one while still achieving sublinear regret. 
In contrast to the full information scenario above (and the one with two prices and partial feedback that we discuss later), we give a negative answer to this question. 
\begin{figure}[t]
    \centering
    \begin{tikzpicture}
    \def\k{0.45}
    \definecolor{myblue}{RGB}{25,175,255}
    \definecolor{myred}{RGB}{255,70,75}
    \def\colorOne{myblue}
    \def\colorTwo{myred}
    \fill[\colorOne] ({0*\k},{3*\k}) rectangle ({1*\k},{4*\k});
    \fill[\colorOne] ({4*\k},{5*\k}) rectangle ({5*\k},{6*\k});
    \fill[\colorOne] ({2*\k},{7*\k}) rectangle ({3*\k},{8*\k});
    \fill[\colorTwo] ({2*\k},{3*\k}) rectangle ({3*\k},{4*\k});
    \fill[\colorTwo] ({4*\k},{7*\k}) rectangle ({5*\k},{8*\k});
    \fill[\colorTwo] ({0*\k},{5*\k}) rectangle ({1*\k},{6*\k});
    \draw ({0*\k}, {8*\k}) -- ({8*\k}, {8*\k}) -- ({8*\k}, {0*\k});
    \draw[gray, dotted] ({0*\k}, {8*\k}) -- ({8*\k}, {0*\k});
    \draw[->] ({-2*\k}, {0*\k}) -- ({11*\k}, {0*\k}) node[below right] {$s$};
    \draw[->] ({0*\k}, -{1*\k}) -- ({0*\k}, {8.5*\k}) node[above left] {$b$};
    \draw (0,0) node[below left] {$0$}
        ({\k*8, 0}) node[below] {$1$}
        ({0, \k*8}) node[left] {$1$}
        ;
        \draw 
            ({\k/2}, {\k*7/2}) node {$Q_1$};
        \draw 
            ({\k*5/2}, {\k*15/2}) node {$Q_2$};
        \draw 
            ({\k*9/2}, {\k*11/2}) node {$Q_3$};
        \draw 
            ({\k*9/2}, {\k*15/2}) node {$Q_4$};
        \draw 
            ({\k*1/2}, {\k*11/2}) node {$Q_5$};
        \draw 
            ({\k*5/2}, {\k*7/2}) node {$Q_6$};
    \end{tikzpicture}
    \caption{The squares $Q_1,\dots,Q_6$ appearing in the proof of \Cref{thm:two-bit-one-price-lower}.}
    \label{fig:proof-lb-linear}
\end{figure}
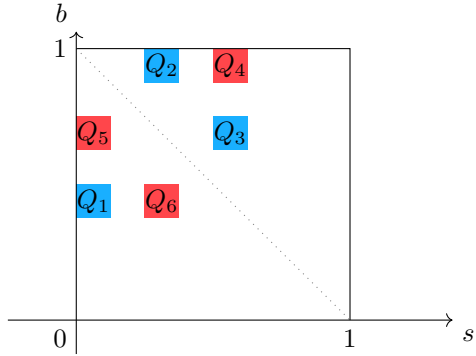
\begin{theorem}
\label{thm:two-bit-one-price-lower}
    Consider the problem of repeated bilateral trade against a $\sigma$-smooth adversary in the two-bit feedback model, for any $\sigma \le \tfrac{1}{64}$. Then any learning algorithm that posts a single price per time step suffers at least $\tfrac{T}{24}$ regret, even if $S_1, B_1, S_2, \dots$ is an independent family of random variables.
\end{theorem}
\begin{proof}
Consider the following six squares, depicted in \Cref{fig:proof-lb-linear}:
\begin{align*}
    Q_1 
& =  
    \lsb{ 0, \,\frac{1}{8} } \times \lsb{ \frac{3}{8}, \,\frac{1}{2} }
\;, &  
    Q_2
& =  
    \lsb{ \frac{1}{4}, \,\frac{3}{8} } \times \lsb{ \frac{7}{8}, \,1}
\;, &
    Q_3
& =
    \lsb{ \frac{1}{2}, \,\frac{5}{8} } \times \lsb{ \frac{5}{8}, \, \frac{3}{4} } \;,
\\
    Q_4 
& =  
    \lsb{ \frac{1}{2}, \,\frac{5}{8} } \times \lsb{ \frac{7}{8}, \,1}
\;, &  
    Q_5
& =  
    \lsb{ 0, \,\frac{1}{8} } \times \lsb{ \frac{5}{8}, \,\frac{3}{4}}
\;, &
    Q_6
& =
    \lsb{ \frac{1}{4}, \,\frac{3}{8} } \lsb{ \frac{3}{8}, \,\frac{1}{2} } \;,
\end{align*}
To each square $Q_i$, we associate a uniform probability distribution over it: we say that the random valuations $(S,B)$ are distributed uniformly over $Q_i$ under $\Pb^i$ and $\mathbb{E}^i$, for each $i = 1,\dots,6$. Starting from these distributions, we construct two other distributions: the ``red'' one and the ``blue'' one. When $(S,B)$ is sampled from the blue one, it is sampled u.a.r. from the union of the blue squares: ($Q_1,Q_2$ and $Q_3$). In formula, the probability measure $\Pb^{\textrm{blue}}$ is just a uniform mixture of $\Pb^1$, $\Pb^2$ and $\Pb^3$. The same can be done for the red distribution over the red squares ($Q_4,Q_5$ and $Q_6$). Note that both the red and the blue distributions are $\frac{1}{64}$ smooth.

From \citet[Theorem~4.3]{Nicolo21}, we know that any learning algorithm $\cA$ that can only post one price $P_t$ suffers linear regret against at least one of the following i.i.d. instance: the adversary chooses at the beginning of time either the red or the blue distribution and extracts valuations from it i.i.d. over the rounds. In formula:
\begin{equation}
\label{eq:oldlower}
    \max_{\textrm{color} \in  \{\textrm{blue},\textrm{red}\}} \bbrb{ \max_{p \in [0,1]} \sum_{t=1}^T \bbE^{\textrm{color}} \bsb{ \gft_t(p) - \gft_t(P_t) } }
\ge
    \frac1{24}T.
\end{equation}
We cannot use directly this construction for our result, as seller and buyer valuations are not independent in the blue and red distributions. However, we can exploit the non i.i.d. structure of the smooth adversary, to generate an equivalent random sequence of smooth distributions such that each one of them has {\em independent} seller and buyer valuations.  

Consider the following family $F$ of $1/64$-smooth oblivious adversaries: each $\cS$ of them is characterized by a color red or blue, and a sequence $\{i_t\}$ of $T$ indices, where red adversaries have $i_t \in \{4,5,6\}$ and blue adversaries have $i_t \in \{1,2,3\}.$ We denote with $F^{\textrm{red}}$ the set of all such adversaries and with $F^{\textrm{blue}}$ the blue ones. Any $\cS$ in the sequence generates the valuations as follows: $(S_t,B_t)$ is drawn independently and uniformly at random from $Q_{i_t}$. Note that any $\cS \in F$ 
 enjoys the property that the distribution chosen at each time step has independent seller and buyer. We argue that any learning algorithm $\cA$ suffers linear regret against at least one of these adversaries. In formula:
\begin{align}
    \nonumber
    R_T(\cA) &\ge \max_{\cS \in F}\left[\max_{p\in[0,1]}\left(\sum_{t=1}^T \mathbb{E}^{i_t}\left[\gft_t(p) - \gft_t(P_t)\right]\right)\right]\\
    \nonumber
    &= \max_{\textrm{color} \in \{\textrm{red},\textrm{blue}\}}\max_{\cS \in F^{\textrm{color}}}\left[\max_{p\in[0,1]}\left(\sum_{t=1}^T \mathbb{E}^{i_t}\left[\gft_t(p) - \gft_t(P_t)\right]\right)\right]
    \\
    \label{eq:magic}
    &\ge \max_{\textrm{color} \in \{\textrm{red},\textrm{blue}\}}\left[\max_{p\in[0,1]}\left(\sum_{t=1}^T \mathbb{E}^{\textrm{color}}\left[\gft_t(p) - \gft_t(P_t)\right]\right)\right]
\end{align}
Note that the $i_t$ are the indices induced by $\cS$. The previous inequality, combined with \Cref{eq:oldlower} concludes the proof. The only delicate step we need to clarify is the last inequality in \Cref{eq:magic}. To this end, fix any color, let's say red (same argument holds for blue). 
The regret of $\cA$ against the worst sequence in $F^{\textrm{red}}$ is at least the expected regret of $\cA$ against a randomized adversary that is obtained by drawing u.a.r. $\cS$ from $F^{\textrm{red}}$ (note that the adversaries in $F^{\textrm{red}}$ are oblivious).
Now, the crucial argument is that the sequence of valuations $(S_t,B_t)$ obtained by choosing u.a.r. an adversary $\cS$ from $F^{\textrm{red}}$ follows the exact same distribution as drawing $(S_t,B_t)$ i.i.d. from the red distribution. In fact, the valuations at different steps are independent and every square has the same probability of being chosen at each time step.
\end{proof}

\section{A \texorpdfstring{$T^{3/4}$}{Omega(T 3/4)} lower bound: two bits and two prices}
\label{sec:lowerbound}

In this section, we present the main contribution of this paper: an unexpected and intriguing lower bound of order $T^{3/4}$. This result has two notable implications.
First, it provides a formalization to the intuition that \emph{partial} feedback (both one- and two-bit models) is strictly less informative than the \emph{bandit} feedback, being the regret of the latter of order at most $T^{2/3}$.
Second, noting that the hard instances constructed in the proof of \Cref{thm:two-bit-two-prices-lower} are i.i.d., we solve an open problem in \cite{Nicolo21}, where it was erroneously conjectured that the correct minimax rate was $T^{2/3}$.

\begin{theorem}
\label{thm:two-bit-two-prices-lower}
Consider the problem of repeated bilateral trade against a $\sigma$-smooth adversary in the two-bit feedback model, for any $\sigma \le \tfrac 19$. If $T \ge 8008$, then any learning algorithm $\cA$ that posts two prices per time step suffers at least a regret of
\[
    R_T(\cA)
\ge
    \frac{1}{50^3}T \;.
\]
\end{theorem}
\begin{proof}
We prove this result in several steps: in \Cref{proof:step1} we begin by constructing a hard instance of the repeated bilateral trade learning problem, then in \Cref{proof:step2} we present a related (easier) learning problem (\gener{})and, finally, in \Cref{proof:step3} we show that the minimax regret of the latter (and therefore, the former) is at least $\Omega(T^{3/4})$. The proofs of some of the intermediate Lemmata are postponed to \Cref{app:lowerbound}.

    \subsection{The construction of a hard family of adversaries}
    \label{proof:step1}
    
    Fix any $\sigma \in(0, \nicefrac19]$ and $T \ge 8008$. Our lower bound is based on a family of i.i.d. $\sigma$-smooth adversaries for the repeated bilateral trade learning problem. Under each such adversary the valuations $(S_t, B_t)$ are drawn i.i.d.\ according to a fixed distribution, obliviously of the actions of the learner.\footnote{This assumption makes our result even stronger: restricting the adversary can never make the lower bound bigger.} Formally, for any algorithm $\cA$ of the learner, we prove the existence of an i.i.d. $\sigma$-smooth adversary such that the following lower bound on the regret holds:
    \[
        \max_{p\in [0,1]} T \cdot \bbE\bsb{ \gft(p,S,B)}
        -
        \bbE \lsb{ \sum_{t=1}^T \gft(P_t,Q_t,S_t,B_t) }
        \ge
        \frac{1}{50^3} T^{3/4} \;.
    \]
    For the sake of convenience, each adversary is identified by a probability measure over $[0,1]^2$, from which the random valuations are drawn. These probability measures are absolutely continuous with respect to the Lebesgue measure and are obtained by suitable perturbations over a base distribution $f$, whose support is given by 
    the union of the six squares $Q_1, \dots, Q_6$ (see \Cref{fig:LB-six-squares}, left): 
    \begin{align*}
    Q_1 & = \textstyle{ \lsb{0, \frac{1}{6}} \times \left[ \frac{1}{3}, \frac{1}{2} \right) } \;, 
    & Q_2 & = \textstyle{ \lsb{0, \frac{1}{6}} \times \lsb{ \frac{1}{2}, \frac{2}{3} } } \;,
    & Q_3 & = \textstyle{ \lsb{ 0, \frac16 } \times \lsb{ \frac56, 1 } } \;,
    \\
    Q_4 & = \textstyle{ \lsb{ \frac56, 1 } \times \lsb{ \frac56, 1 } } \;,
    & Q_5 & = \textstyle{ \lsb{ \frac56, 1 } \times \lsb{ 0, \frac16 } } \;,
    & Q_6 & = \textstyle{ \lsb{ \frac13, \frac12 } \times \lsb{ \frac23, \frac56 } } \;.
    \end{align*}
\def\v{1/3 + 3/48}%
\def\eps{1/48}%
\def\vme{\v - \eps}%
\def\vpe{\v + \eps}%
\def\cOne{2.5}%
\def\cTwo{30}%
\def\cThree{20}%
\begin{figure}
\centering
\begin{tikzpicture}[scale = 3]
\draw[gray!25!white, very thin] 
({\vme},4/6) -- ({\vme}, 0)
({\v},4/6) -- ({\v}, 0)
({\vpe},4/6) -- ({\vpe}, 0)
;
\draw [very thin, mygreen!50!black, fill = mygreen]
    (0,1/3) rectangle (1/6,1/2)
    (0,1/2) rectangle (1/6,2/3)
    (0,5/6) rectangle (1/6,1)
    (5/6,5/6) rectangle (1,1)
    (5/6,0) rectangle (1,1/6)
    (1/3,2/3) rectangle ({\vme},5/6)
    ({\vpe},2/3) rectangle (1/2,5/6)
;
\draw [very thin, myblue!50!black, fill = myblue]
    ({\vme},2/3) rectangle ({\v},3/4)
    ({\v},3/4) rectangle ({\vpe},5/6)
;
\draw [very thin, myred!50!black, fill = myred]
    ({\vme},3/4) rectangle ({\v},5/6)
    ({\v},2/3) rectangle ({\vpe},3/4)
;
\draw (0,1) -- (1,1) -- (1,0);
\draw [<->] (0,1.1) -- (0,0) -- (1.1,0);
\draw (0, {(2/6+3/6)/2}) node[left] {$Q_1$}
    (0, {(3/6+4/6)/2}) node[left] {$Q_2$}
    (0, {(5/6+6/6)/2}) node[left] {$Q_3$}
    ({(5/6+1)/2}, 5/6) node[below] {$Q_4$}
    ({(5/6+1)/2}, {1/6}) node[above] {$Q_5$}
    (3/6, {(4/6+5/6)/2}) node[right] {$Q_6$}
;
\draw (0,0) -- (0,-0.03) node[below] {$0$}
({\v},0) -- ({\v},-0.03) node[below] {$v$}
(2/3,-0.03) node[below,white] {$\vphantom{\nicefrac{2}{3}}$}
(1,0) -- (1,-0.03) node[below] {$1$}
(0,0) -- (-0.03,0)
(0,1) -- (-0.03,1)
;
\draw[very thick, gray] 
(1/3,2/3) rectangle (1/2,5/6)
;
\draw[very thick, gray] 
(1/2,2/3) -- ({1+2/3},0)
(1/2,5/6) -- ({1+2/3},1)
;
\draw[very thin, mygreen, fill = mygreen]
    ({1+2/3},0) rectangle ({1+2/3+12*\eps},1)
    ({1+2/3+24*\eps},0) rectangle ({1+2/3+48*\eps},1)
;
\draw [very thin, myblue, fill = myblue]
({1+2/3+12*\eps},0) rectangle ({1+2/3+18*\eps},1/2)
({1+2/3+18*\eps},1/2) rectangle ({1+2/3+24*\eps},1)
;
\fill [myred]
({1+2/3+12*\eps},1/2) rectangle ({1+2/3+18*\eps},1)
({1+2/3+18*\eps},0) rectangle ({1+2/3+24*\eps},1/2)
;
\draw[very thin, gray]
({1+2/3+12*\eps},1/2)
-- ({1+2/3+24*\eps},1/2)
({1+2/3+12*\eps},1)
-- ({1+2/3+12*\eps},1-1/8)
({1+2/3+12*\eps},1/2+1/7)
-- ({1+2/3+12*\eps},1/2-1/7)
({1+2/3+12*\eps},1/8)
-- ({1+2/3+12*\eps},0)
({1+2/3+18*\eps},1)
-- ({1+2/3+18*\eps},0)
({1+2/3+24*\eps},1)
-- ({1+2/3+24*\eps},1-1/8)
({1+2/3+24*\eps},1/2+1/7)
-- ({1+2/3+24*\eps},1/2-1/7)
({1+2/3+24*\eps},1/8)
-- ({1+2/3+24*\eps},0)
;
\draw[very thick, gray] 
(1+2/3, 0) rectangle (2+2/3, 1)
;
\draw 
({1+2/3+12*\eps},3/4) node {$R^1_{v,\e}$}
({1+2/3+12*\eps},1/4) node {$R^2_{v,\e}$}
({1+2/3+24*\eps},3/4) node {$R^3_{v,\e}$}
({1+2/3+24*\eps},1/4) node {$R^4_{v,\e}$}
(2+2/3, 1/2) node[right] {\LARGE$Q_6$}
;
\end{tikzpicture}
\ 
\begin{tikzpicture}[scale = 3
, 
declare function={
egft(\p,\a)
= 
1 / ( 6*( 1 + 8*\a ) ) * (
    (\p < 0) * (0)
    + and(\p >= 0, \p <= 1/6) 
        * ( 3*\p*( 5 + 29*\a - 6*( 1 + 3*\a )*\p ) )
    + and(\p > 1/6, \p <= 1/2) 
        * ( 2 + 13*\a )
    + and(\p > 1/2, \p <= 2/3) 
        * ( -18*\a*\p*\p + 3*\a*\p + 2*( 1 + 8*\a ) )
    + and(\p > 2/3, \p <= 5/6) 
        * ( -18*\p*\p + 15*\p + 10*\a )
    + and(\p > 5/6, \p <= 1) 
        * ( 72*\a*\p*( 1 - \p ) )
    + (\p > 1) * (0)
)
;
tent(\u,\r,\x)
=
(\x < \u - \r) * (0)
+ and(\x >= \u - \r, \x <= \u) 
        * ( 1 + ( \x - \u )/\r )
+ and(\x > \u, \x <= \u + \r)
    * ( 1 - ( \x - \u )/\r )
+ (\x > \u + \r) * (0)
;
}
]
\draw[gray!25!white, very thin] 
({\vme},{\cOne* egft(1/2, 2*ln(27/16)) + \cTwo* (\eps/12) + 0.1}) -- ({\vme}, 0)
({\v},{\cOne* egft(1/2, 2*ln(27/16)) + \cTwo* (\eps/12)}) -- ({\v}, 0)
({\vpe},{\cOne* egft(1/2, 2*ln(27/16)) + \cTwo* (\eps/12) + 0.1 }) -- ({\vpe}, 0)
(0, {\cOne* egft(1/2, 2*ln(27/16))}) -- (1/6, {\cOne* egft(1/2, 2*ln(27/16))})
(1/2, {\cOne* egft(1/2, 2*ln(27/16))}) -- (1, {\cOne* egft(1/2, 2*ln(27/16))})
(0, {\cOne* egft(1/2, 2*ln(27/16)) + \cTwo* (\eps/12)}) -- (1, {\cOne* egft(1/2, 2*ln(27/16)) + \cTwo* (\eps/12)})
(2/3, {\cOne* egft(2/3, 2*ln(27/16))}) -- (2/3,0)
(2/3, {\cOne* egft(2/3, 2*ln(27/16))}) -- (0,{\cOne* egft(2/3, 2*ln(27/16))})
;
beautified plot
\draw[red] 
    (0,0) -- (1/6,
        {
        \cOne* egft(1/6, 2*ln(27/16))
        + \cTwo* (\eps/12)*tent(\v,\eps,1/6) 
        + \cThree* (\eps*\eps)*tent(3/4,1/12,1/6) 
        }) --
    ({\vme},
        {
        \cOne* egft(1/6, 2*ln(27/16))
        + \cTwo* (\eps/12)*tent(\v,\eps,1/6) 
        + \cThree* (\eps*\eps)*tent(3/4,1/12,1/6) 
        }) --
    ({\v},
        {
        \cOne* egft(\v, 2*ln(27/16))
        + \cTwo* (\eps/12)*tent(\v,\eps,\v) 
        + \cThree* (\eps*\eps)*tent(3/4,1/12,\v) 
        }) --
    ({\vpe},
        {
        \cOne* egft(1/6, 2*ln(27/16))
        + \cTwo* (\eps/12)*tent(\v,\eps,1/6) 
        + \cThree* (\eps*\eps)*tent(3/4,1/12,1/6) 
        }) --
    (1/2,
        {
        \cOne* egft(1/6, 2*ln(27/16))
        + \cTwo* (\eps/12)*tent(\v,\eps,1/6) 
        + \cThree* (\eps*\eps)*tent(3/4,1/12,1/6) 
        }) --
    (2/3,
        {
        \cOne* egft(2/3, 2*ln(27/16))
        + \cTwo* (\eps/12)*tent(\v,\eps,2/3) 
        + \cThree* (\eps*\eps)*tent(3/4,1/12,2/3) 
        }) --
    (3/4,
        {
        \cOne* egft(3/4, 2*ln(27/16))
        + \cTwo* (\eps/12)*tent(\v,\eps,3/4) 
        + \cThree* (\eps*\eps)*tent(3/4,1/12,3/4) 
        }) --
    (5/6,
        {
        \cOne* egft(5/6, 2*ln(27/16))
        + \cTwo* (\eps/12)*tent(\v,\eps,5/6) 
        + \cThree* (\eps*\eps)*tent(3/4,1/12,5/6) 
        }) --
    (1,0);
\draw[densely dotted, thick] 
    (0,0) -- (1/6,
        {
        \cOne* egft(1/6, 2*ln(27/16))
        + \cTwo* (\eps/12)*tent(\v,\eps,1/6) 
        + \cThree* (\eps*\eps)*tent(3/4,1/12,1/6) 
        }) --
    (1/2,
        {
        \cOne* egft(1/6, 2*ln(27/16))
        + \cTwo* (\eps/12)*tent(\v,\eps,1/6) 
        + \cThree* (\eps*\eps)*tent(3/4,1/12,1/6) 
        }) --
    (2/3,
        {
        \cOne* egft(2/3, 2*ln(27/16))
        + \cTwo* (\eps/12)*tent(\v,\eps,2/3) 
        + \cThree* (\eps*\eps)*tent(3/4,1/12,2/3) 
        }) --
    (5/6,
        {
        \cOne* egft(5/6, 2*ln(27/16))
        + \cTwo* (\eps/12)*tent(\v,\eps,5/6) 
        + \cThree* (\eps*\eps)*tent(3/4,1/12,5/6) 
        }) --
    (1,0);
\draw [<->] (0,1.1) -- (0,0) -- (1.1,0) node[right] {$p$};
\draw (0,0) -- (0,-0.03) node[below] {$0$}
({\v},0) -- ({\v},-0.03) node[below] {$v$}
(2/3,0) -- (2/3,-0.03) node[below] {$\nicefrac{2}{3}$}
(1,0) -- (1,-0.03) node[below] {$1$}
(1.01, {\cOne* egft(1/2, 2*ln(27/16)) + \cTwo* (\eps/12)}) 
    -- (1.03, {\cOne* egft(1/2, 2*ln(27/16)) + \cTwo* (\eps/12)}) 
    -- (1.03, {\cOne* egft(1/2, 2*ln(27/16))}) 
    -- (1.01, {\cOne* egft(1/2, 2*ln(27/16))})
(1.03, {\cOne* egft(1/2, 2*ln(27/16)) + 0.5* \cTwo* (\eps/12)}) node[right] {$\Theta(\e)$}
({\vme},{\cOne* egft(1/2, 2*ln(27/16)) + \cTwo* (\eps/12) + 0.11})
-- ({\vme},{\cOne* egft(1/2, 2*ln(27/16)) + \cTwo* (\eps/12) + 0.13}) -- ({\vpe},{\cOne* egft(1/2, 2*ln(27/16)) + \cTwo* (\eps/12) + 0.13}) -- ({\vpe},{\cOne* egft(1/2, 2*ln(27/16)) + \cTwo* (\eps/12) + 0.11})
({\v}, {\cOne* egft(1/2, 2*ln(27/16)) + \cTwo* (\eps/12) + 0.13}) node[above] {$\Theta(\e)$}
(-0.01, {\cOne* egft(1/2, 2*ln(27/16))}) 
    -- (-0.03, {\cOne* egft(1/2, 2*ln(27/16))}) 
    -- (-0.03, {\cOne* egft(2/3, 2*ln(27/16))}) 
    -- (-0.01, {\cOne* egft(2/3, 2*ln(27/16))})
(-0.03, {(\cOne* egft(1/2, 2*ln(27/16)) + \cOne* egft(2/3, 2*ln(27/16)))/2}) node[left] {$\Theta(1)$}
;
\end{tikzpicture}
\caption{\footnotesize Left/center: The six squares $Q_1, \dots, Q_6$ (in green) are the support of the base density $f$, and the four rectangles $R^1_{v,\e},\dots, R^4_{v,\e}$ (in red and blue) inside $Q_6$ are the regions where the density is perturbed with $g_{v,\e}$. 
Right: The corresponding qualitative plots of $p \mapsto \bbE[\gft(p,S,B)]$ (black, dotted) and $p \mapsto \bbE^{v,\e}[\gft(p,S,B)]$ (red, solid).%
}
\label{fig:LB-six-squares}
\end{figure}
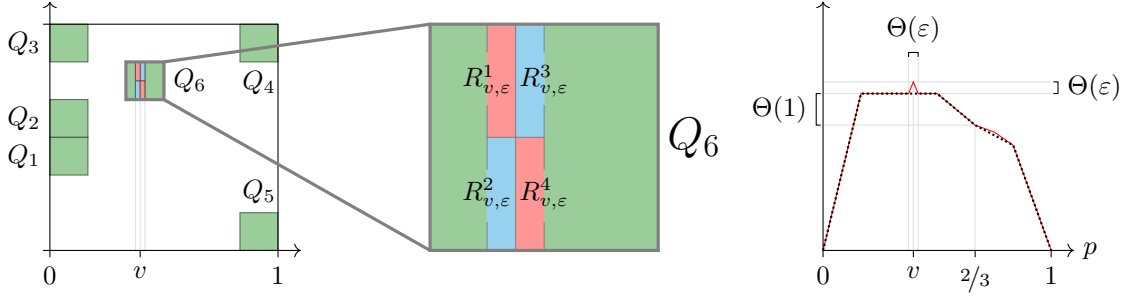%
    The base probability density function $f$ is defined for all $(x,y) \in [0,1]^2$ by
    \[
        f(x,y)
    =
        \frac{36}{1+8a} \cdot \lrb{
            \frac{ 5 - 6 (y+x) }{ 6(y-x) } \I_{Q_1}(x,y)
        + a \I_{ Q_2 }(x,y)
        + 2a \I_{ Q_3 \cup Q_4 \cup Q_5 } (x,y)
        + \I_{ Q_6 } (x,y)
        },
    \]
    where $a$ is set to $2 \cdot \ln(27/16)$ for normalization. Each perturbations is parametrized by a center $v$ and a scale $\e$, with $(v,\e) \in
        \perturb
    =
        \textstyle{
        \Bcb{ 
        (v,\e) \in \brb{\frac13,\frac12} \times \brb{ 0, \frac{1}{12} }
        \mid 
        \frac13 + \e \le v \le \frac12 - \e
        }.
    }$
    For all $(v,\e)$, we have four disjoint rectangles (\Cref{fig:LB-six-squares}, left): 
    \begin{align*}
        R^1_{v, \e} 
    & =
    \textstyle{
        \left[ v - \e, v \right) \times \lsb{ \frac34, \frac56 }
    }
    \;,
       & R^2_{v, \e} &
    =
    \textstyle{
        \left[ v - \e, v \right) \times \left[ \frac23, \frac34 \right)
    }
    \;,
    \\
        R^3_{v, \e} &    
    =
    \textstyle{
        \left[ v , v + \e \right] \times \lsb{ \frac34, \frac56 }
    }
    \;,
        & R^4_{v, \e} &
    =
    \textstyle{
        \left[ v , v + \e \right] \times \left[ \frac23, \frac34 \right) 
    \;.
    }
    \end{align*}
    The $R^i_{v, \e}$  rectangles are included in $Q_6$ and are the support of the corresponding perturbation $g_{v,\e}$ defined for all $(x,y) \in [0,1]^2$ by
    \[
        g_{v,\e}(x,y)
    =
        \frac{36}{1+8a} \cdot 
    \lrb{
        \I_{R^1_{v, \e} \cup R^4_{v, \e}}(x,y)
        - \I_{R^2_{v, \e} \cup R^3_{v, \e}}(x,y)
    } \;.
    \]
    The perturbed density functions are obtained by summing together the base probability density function $f$ and one of the perturbations $g_{v,\e}$.
    Formally, for all $(v,\e) \in \perturb$, we let $
        f_{v,\e} = f + g_{v,\e}.$

    Let $\Pb$ (resp., $\Pb^{v,\e}$, for all $(v,\e) \in \perturb$) be a probability measure such that the sequence of seller/buyer evaluations $(S,B)$, $(S_1,B_1)$, $(S_2,B_2)$, $\dots$ is i.i.d. and the distribution of $(S,B)$ has density $f$ (resp., $f_{v,\e}$) with respect to the Lebesgue measure.
    We denote the expectation with respect to $\Pb$ (resp., $\Pb^{v,\e}$, for all $(v,\e) \in \perturb$) by $\bbE$ (resp., $\bbE^{v,\e}$).
    Note that $\Pb_{(S,B)}$ (resp., $\Pb_{(S,B)}^{v,\e}$, for all $(v,\e) \in \perturb$) is $\nicefrac{1}{9}$-smooth (hence, each adversary corresponding to these distributions is $\sigma$-smooth for any $\sigma \le \tfrac 19$).

    \paragraph{Expected gain from trade.} The particular choice of distributions $f_{v,\e}$ is due to the specific structure of the expected gain from trade and feedback they induce. We start analyzing the former. For each $(v,\e) \in \perturb$, and $p \in [0,1]$, it is easy to argue, by linearity, that 
    \begin{align}
        \nonumber
            \bbE^{v,\e} \bsb{ \gft (p,S,B) }
        &=
            \bbE \bsb{ \gft (p,S,B) } + \int_{[0,p]\times[p,1]}(y-x) g_{v,\e}(x,y) \, \mathrm{d}x\mathrm{d}y
        \\
        \label{eq:gft_pert}
        &=
            \bbE \bsb{ \gft (p,S,B) } + \tfrac{\e}{864(1+8a)} \cdot \Lambda_{v,\e}(p) + \tfrac{\e^2}{72 (1+8a)} \cdot \Lambda_{\frac{3}{4},\frac{1}{12}}(p) \;,
    \end{align}
    where $\Lambda_{u,r}$ is the tent map centered at $u$ with radius $r$,  $\Lambda_{u,r}(x)= \lrb{1-\frac{|x-u|}{r}}^+$. \Cref{eq:gft_pert} nicely decomposes the expected gain from trade in a fixed term that depends only on the base distribution, a perturbation term centered in $v$ and a second order $\Theta(\e^2)$ term. Simple calculations yields the  analytical expression of $\E{\gft(p,S,B)}$:
    \begin{equation}
    \label{e:gft-base}
        \bbE \bsb{ \gft (p,S,B) } = \frac{1}{6 (1+8a)} \cdot
    \begin{cases}
        3p\brb{ 5 + 29a - 6(1+3a)p} &\text{if $p \in \bsb{ 0,\frac{1}{6} }$}\\
        2 + 13a &\text{if $p \in \bigl(\frac{1}{6},\frac{1}{2}\bigr]$}\\
        -18ap^2 + 3ap + 2(1+8a) &\text{if $p \in \bigl(\frac{1}{2},\frac{2}{3}\bigr]$}\\
        -18p^2 + 15p + 10a &\text{if $p \in \bigl(\frac{2}{3},\frac{5}{6}\bigr]$}\\
        72ap(1-p) &\text{if $p \in \bigl(\frac{5}{6},1\bigr]$}\\
    \end{cases}\;
    \end{equation}
    To have a qualitative understanding of \Cref{e:gft-base} we refer to \Cref{fig:LB-six-squares} (dotted black plot on the right): it is clear that the maximum is attained in the plateau corresponding to $p\in [\tfrac 16, \tfrac12 ]$. Furthermore, for each $(v,\e) \in \perturb$, price $v$ is the unique maximizer of the perturbed expected gain from trade $\bbE^{v,\e}\bsb{ \gft (p,S,B) }$, which is increasing on $\bsb{ 0, \frac16}$, constant on $\bsb{ \frac16, v-\e}$, has a symmetric spike on $[v-\e,v+\e]$, becomes constant again on $\bsb{ v+\e, \frac12 }$, and decreases on $\bsb{ \frac12, 1}$(\Cref{fig:LB-six-squares}, red plot on the right). Recall that there is no point in posting two different prices to maximize the gain from trade, thus the expected gain from trade under $\bbE^{v,\e}$ is maximized by posting $(v,v)$. $\U$ denotes the upper triangle: the set of all the feasible pair of prices the learner can post. It holds: 
    \[
        \max_{(p,q) \in \U} \bbE^{v,\e} \bsb{ \gft(p,q,S,B) } = \bbE^{v,\e} \bsb{ \gft(v,S,B) }.
    \]

    \paragraph{Two-bit feedback.} We move our attention to the description of the distribution of the 2-bit feedback $\lrb{ \I \bcb{S \le p}, \I\bcb{q \le B}}$. It is the same regardless of the underlying perturbed probability measure unless the learner selects a pair of prices $(p,q)$ in one of the four rectangles $R^i_{v,\e}$ where the perturbations occur. For the sake of simplicity, we use the random variable $Z$ to denote $\lrb{ \I \bcb{S \le p}, \I\bcb{q \le B}}$.
\begin{claim}
\label{cl:perturbation}
    Fix any $(v,\e)\in \perturb$, $(p,q) \in \cU \setminus \bigcup_{i \in [4]} R_{v,\e}^i$, and let $Z = ( \I \{ S \le p \}, \I \{q \le B\} )$.
    Then $Z$ follows the same distribution both under $\mathbb P$ and $\mathbb P^{v,\e}$. Formally, the following holds
    \[
    \Pb^{v,\e}
    \Bsb{ Z = (i,j) }
    =
    \Pb
    \Bsb{  Z= (i,j) } \quad \forall \, (i,j) \in \{0,1\}^2.
    \]
\end{claim}
\begin{proof}
For each $(v,\e) \in \perturb$, and each $(p,q) \in \cU$, the distribution under $\Pb^{v,\e}$ of the 2-bit feedback $Z$ is given by:
\begin{itemize}
    \item $\P{Z = (0,0)} =  \Pb^{v,\e} \bsb{ S>p \cap B < q } = \int_p^1 \int_0^q f(x,y) \dif x \mathrm{d}y + \int_p^1 \int_0^q g_{v,\e} (x,y) \dif x \mathrm{d}y$
    \item $\P{Z = (0,1)} = \Pb^{v,\e} \bsb{ S>p \cap B \ge q }= \int_p^1 \int_q^1 f(x,y) \dif x \mathrm{d}y + \int_p^1 \int_q^1 g_{v,\e} (x,y) \dif x \mathrm{d}y$
    \item $\P{Z = (1,0)} = \Pb^{v,\e} \bsb{ S\le p \cap B < q }= \int_0^p \int_0^q f(x,y) \dif x \mathrm{d}y + \int_0^p \int_0^q g_{v,\e} (x,y) \dif x \mathrm{d}y$
    \item $\P{Z = (1,1)} = \Pb^{v,\e} \bsb{ S\le p \cap B \ge q }= \int_0^p \int_q^1 f(x,y) \dif x \mathrm{d}y + \int_0^p \int_q^1 g_{v,\e} (x,y) \dif x \mathrm{d}y$.
\end{itemize}
By symmetry, all integrals of $g_{v,\e}$ in the previous formulae vanish if $(p,q)$ does not belong to one of the four rectangles $R^1_{v, \e}, R^2_{v, \e}, R^3_{v, \e}, R^4_{v, \e}$. The desired claim follows from the fact that $(p,q) \notin Q_6 \supseteq R^1_{v, \e} \cup R^2_{v, \e} \cup R^3_{v, \e} \cup R^4_{v, \e}$.
\end{proof}

\paragraph{The cost of exploration and of suboptimality.} In the following of the proof, we use $K = \lce{T^{1/4} }$, and $\e = \frac1{12K}$. For each $k \in \{1,\dots, K\}$, let $v_k = \frac{1}{3} + (2k-1) \e$ be a candidate center. For the sake of convenience, for each $k \in [K]$ denote $\Pb^{v_k, \e}$ by $\Pb^k$ and the corresponding expectation by $\bbE^k$, and similarly, denote $\Pb$ by $\Pb^0$ and the corresponding expectation by $\bbE^0$. Our construction has two crucial features: first, posting prices that are $\e$ far from the optimal one yields $\Theta(\e)$ instantaneous regret (this is immediate from the analytic expression of expected gain from trade and is formalized in \Cref{cl:cost_sub}); second, the learner, to locate the actual perturbation, is forced to post prices in a suboptimal regiorn $Q_6$ which incurs in constant instantaneous regret (\Cref{cl:cost_expl}).
\begin{claim}[Cost of suboptimality]
\label{cl:cost_sub}
    Fix any perturbation pair $(v,\e) \in \perturb$ and let $p$ be any price not in $[v-\e,v+\e]$, then the following holds:
    \[
        \bbE^{v,\e} \bsb{ \gft (v,S,B ) }
-
    \bbE^{v,\e} \bsb{ \gft ( p,S,B ) }
\ge
    \e \cdot \frac{1}{864(1+8a)}      
    \]
\end{claim}

\begin{claim}[Cost of exploration]
\label{cl:cost_expl}
    Fix any perturbation pair $(v,\e) \in \perturb$. Let $(p,q) \in Q_6$ and $p' \in [\tfrac 13, \tfrac 12]$, the following holds:
    \[
        \bbE^{v,\e} \bsb{ \gft (p',S,B ) }
-
    \bbE^{v,\e} \bsb{ \gft ( p,q,S,B ) }
\ge
    \frac{a}{2(1+8a)}
\in [0.05, 0.06] = \Theta(1)        
    \]
\end{claim}
\begin{proof}
    Fix any perturbation pair $(v,\e) \in \perturb$ and consider any pair of prices $(p,q)\in Q_6$. The expected gain from trade corresponding to $(p,q)$ is dominated by the one attainable by posting $(\tfrac12,\tfrac23)$ (each pair of valuations $(s,b)$  such that a trade happens for $(p,q)$ also yields a trade under $(\tfrac12,\tfrac23)$ ). Thus the following holds:
\[
    \bbE^{v,\e} \bsb{ \gft(p,q,S,B) }
\le
    \bbE^{v,\e} \lsb{ \gft \lrb{ \tfrac12,\tfrac23,S,B } }
\le
    \bbE^{v,\e} \lsb{ \gft \lrb{\tfrac23,S,B } }.
\]
On the other hand, posting any pair of prices $(p',p')$ for $p'$ belonging to the potentially optimal region $\bsb{ \tfrac13, \tfrac12 }$ returns
\[
    \bbE^{v,\e} \bsb{ \gft (p',S,B) }
\ge
    \bbE^{v,\e} \lsb{ \gft \lrb{ \tfrac12,S,B } }
\;.
\]
Putting the two inequalities together we get the claimed bound:
\begin{align*}
    \bbE^{v,\e} \bsb{ \gft (p',S,B ) }
-
    \bbE^{v,\e} \bsb{ \gft ( p,q,S,B ) }
&\ge
    \bbE^{v,\e} \lsb{ \gft \lrb{\tfrac12,S,B } }
    -
    \bbE^{v,\e} \lsb{ \gft \lrb{\tfrac23,S,B } }
    \\
&=
    \frac{a}{2(1+8a)}
\in [0.05, 0.06]. \qedhere
\end{align*}
\end{proof}

    \subsection{The \gener{} problem}
    \label{proof:step2}
    Before proceeding further, let's recap what we have obtained so far and where we plan to go. At a high level, we constructed a family of $K$ i.i.d. $\sigma$-smooth adversaries for the repeated bilateral trade problem, each characterized by a probability measure $\Pb^k$ where the valuations $(S,B)$ are sampled from. 
    Under each $\Pb^k$ the expected gain from trade is maximized in a different pair of prices $(v_k,v_k)$.
    Every time the learner posts a price that is $\Omega(\e)$ far from the optimal $v_k$ it suffers instantaneous regret that is $\Omega(\e)$ (\Cref{cl:cost_sub}). To identify the optimal $v_k$, the learner needs to identify the actual perturbation. There are $K = \Theta(\frac 1\e)$ different possible perturbations and, due to the feedback structure, the learner needs to probe separately $K$ disjoint rectangles in $Q_6$ to identify the actual perturbation it is playing against. Recall, posting prices in the suboptimal region $Q_6$ leads to a constant instanatneous regret (\Cref{cl:cost_expl}).

    To better analyze this underlying structure of our construction, we introduce a cleaner discrete problem (that we call \gener{}) on $2K$ actions. The $2K$ actions are of two types: half of them are ``exploring'' actions, while the others are ``exploiting''. At a high level, each exploring action corresponds to one of the exploring rectangles in $Q_6$ for the repeated bilateral trade problem, while each exploiting action is related to a possibly optimal $v_k$. We consider a family of $K$ i.i.d. adversaries for \gener{}, each corresponding to a probability measure $\Pb^1, \dots, \Pb^K$ over the reward vectors.
    Each exploring action $i \in [K]$ yields zero reward but, if played at time $t \in \mathbb{N}$, reveals the reward incurred by the corresponding exploiting action $K+i$. Conversely, exploiting actions yield no feedback\footnote{The reader familiar with the notion of online learning with directed feedback graphs \citealt{AlonCDK15} can see that the feedback model described here corresponds to the weakly observable feedback graph in \Cref{fig:feedback-graph}, left.}. Under each $\Pb^k$, the optimal action is $K+k$, which is $\e$ apart from all the other exploiting actions. 
    The only way the learner has to locate the optimal action is to probe sequentially the exploring actions. Assessing whether a given exploiting arm is the optimal one requires playing $\Omega(\frac 1{\e^2}
    )$ times the corresponding exploring action, which yields constant instantaneous regret. On the other hand, playing a suboptimal action yields $\e $ instantaneous regret. All in all any learner suffers a regret of order $\Omega\brb{\min \brb{ \frac{K}{\e^2}, \e T } } = \Omega(T^{3/4})$, given our choices of $K$ and $\e$.

\begin{figure}
\centering
\begin{tikzpicture}[->]
\draw[white] (0,0) -- (3,3);
\foreach \x in {1,...,4}
{
    \node[draw, circle, inner sep = 2pt] (\x) at (\x*0.85,2.8) {$\x$};
    \path (\x) edge [loop above] (\x);
}
\foreach \x [evaluate={\y=\x-4;}] in {5,...,8}
{
    \node[draw, circle, inner sep = 2pt] (\x) at (\y*0.85,1.25) {$\x$};
}
\path (1) edge (5);
\path (2) edge (6);
\path (3) edge (7);
\path (4) edge (8);
\path (1) edge [loop above] (1);
\end{tikzpicture}
\hspace{2cm}
\definecolor{myColOne}{RGB}{255,0,0}
\definecolor{myColTwo}{RGB}{0,255,0}
\definecolor{myColThree}{RGB}{0,0,255}
\definecolor{myColFour}{RGB}{255,0,255}
\begin{tikzpicture}[scale = 3]
\draw [<->] (0,1.1) -- (0,0) -- (1.1,0);
\draw (0,0) -- (0,-0.03) node[below] {$0$}
    (1,0) -- (1,-0.03) node[below] {$1$}
    (0,0) -- (-0.03,0) node[left] {$0$}
    (0,1) -- (-0.03,1) node[left] {$1$}
;
\fill[myColOne] ({1/3}, 4/6) rectangle ({1/3+2*\eps}, 5/6);
\fill[myColTwo] ({1/3+2*\eps}, 4/6) rectangle ({1/3+4*\eps}, 5/6);
\fill[myColThree] ({1/3+4*\eps}, 4/6) rectangle ({1/3+6*\eps}, 5/6);
\fill[myColFour] ({1/3+6*\eps}, 4/6) rectangle (1/2, 5/6);
\fill[myColOne!25!white] ({1/3}, 4/6) -- ({1/3+2*\eps}, 4/6) -- ({1/3+2*\eps}, {1/3+2*\eps}) -- ({1/3}, {1/3}) -- cycle;
\fill[myColTwo!25!white] ({1/3+2*\eps}, 4/6) -- ({1/3+4*\eps}, 4/6) -- ({1/3+4*\eps}, {1/3+4*\eps}) -- ({1/3+2*\eps}, {1/3+2*\eps}) -- cycle;
\fill[myColThree!25!white] ({1/3+4*\eps}, 4/6) -- ({1/3+6*\eps}, 4/6) -- ({1/3+6*\eps}, {1/3+6*\eps}) -- ({1/3+4*\eps}, {1/3+4*\eps}) -- cycle;
\fill[myColFour!25!white] (0,0) -- (1/3,1/3) -- (1/3,5/6) -- (1/2,5/6) -- (1/2,2/3) -- ({1/2-2*\eps},2/3) -- ({1/2-2*\eps},{1/2-2*\eps}) -- (1,1) -- (0,1) -- cycle;
\draw [very thin, gray] (0,0) -- (1,1) -- (0,1) -- cycle;
    \draw [very thin, gray] (1/3, 2/3) rectangle (1/2, 5/6);
    \draw [very thin, gray]
        (1/3, 2/3) -- (1/3, 1/3)
        ({1/3+2*\eps}, 5/6) -- ({1/3+2*\eps}, {1/3+2*\eps})
        ({1/3+4*\eps}, 5/6) -- ({1/3+4*\eps}, {1/3+4*\eps})
        ({1/3+6*\eps}, 5/6) -- ({1/3+6*\eps}, {1/3+6*\eps})
;
\draw[white] 
    ( {1/3 + \eps}, 3/4 ) node {\tiny$1$}
    ( {1/3 + 3*\eps}, 3/4 ) node {\tiny$2$}
    ( {1/3 + 5*\eps}, 3/4 ) node {\tiny$3$}
    ( {1/3 + 7*\eps}, 3/4 ) node {\tiny$4$}
;
\draw
    ( {1/3 + \eps}, {(1/3 + 7*\eps + 2/3)/2} ) node {\tiny$5$}
    ( {1/3 + 3*\eps}, {(1/3 + 7*\eps + 2/3)/2} ) node {\tiny$6$}
    ( {1/3 + 5*\eps}, {(1/3 + 7*\eps + 2/3)/2} ) node {\tiny$7$}
    ( {1/3 + 7*\eps}, {(1/3 + 7*\eps + 2/3)/2} ) node {\tiny$8$}
;
\end{tikzpicture}
\caption{Left: The feedback graph of \gener{} for $K=4$. Right: The map $\iota$.}
\label{fig:feedback-graph}
\end{figure}
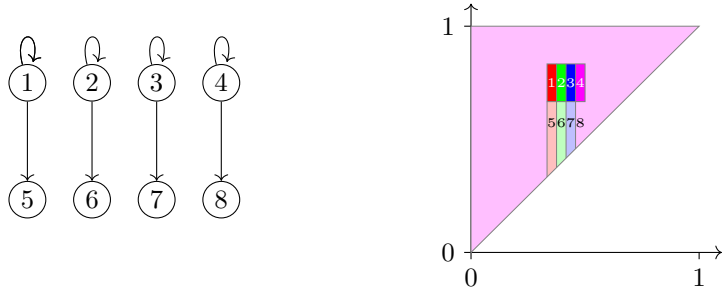

\paragraph{The rewards.} We describe the random rewards of $K+1$ instances of the \gener{} problem associated to $K+1$ probability measures $\mathbb{P}^0,\dots,\Pb^K$. 
Set the parameter $\cprob = \tfrac 3{4a}$, and consider the i.i.d.\ sequence of random vectors $Y,Y_1,Y_2, \dots, Y_T$ such that $Y\in \{0,1\}^{2K}$ and, for each $k\in\{ 0, \dots, K \}$ it holds that
\begin{equation}
    \label{eq:Yprob}
    \Pb^k \bsb{ Y(i) = 1 } = 
        \begin{cases}
            \frac{1}{2} &\text{if $i \in [K] \setminus \{k+K\}$}\\
            \frac{1}{2} + \cprob \cdot \e &\text{if $ i = k + K$}\\
            0 &\text{if $i \in [2K]\setminus [K]$}
        \end{cases}
\end{equation}
    The random vectors $Y_1, Y_2, \dots, Y_T$ are responsible for the reward the learner gets in this new problem.
Formally, a learner that plays action $i \in [2K]$ at time $t$ gets reward $\rho_t(i) = \rho(i,Y_t)$ where
\[
    \rho(i,y) = 
    \begin{cases}
            0 &\text{if $i \in [K]$}\\\cplat + \frac{\cspike}{\cprob} \cdot \brb{y(i-K) - \frac{1}{2}}  &\text{ otherwise}
    \end{cases}
\]
Where we introduced the following two parameters 
\[
    \cplat = \frac{a}{2 (1+8a)}, \, \, \cspike = \frac{1}{864 (1+8a)}.
\]
Note, for any $i\in [K]$, we denote the $i$-th component of $y$ by $y(i)$. Finally, observe that for all $k\in\{0,\dots,K\}$, we have
\begin{equation}
\label{eq:revenue}
\bbE^k\bsb{ \rho(i,Y) }
=
\begin{cases}
0 & \text{ if } k \in [K]\\ 
\cplat + \cspike \cdot \e & \text{ if } k = i-K\\
\cplat & \text{ otherwise }
\end{cases}    
\end{equation}
Before moving on, we spend some word to motivate our choice of $\cplat$ and $\cspike$, whose goal is to relate \Cref{eq:revenue} with the properties of the expected gain from trade in \Cref{e:gft-base,eq:gft_pert}. In particular, $\cplat$ is exactly equal to the cost of exploration term in \Cref{cl:cost_expl}, while $\cspike$ corresponds to the cost of suboptimality, as in \Cref{cl:cost_sub}.

\paragraph{The feedback.} The learner in \gener{} receives two types of feedback. If action  $i \ge K+1$ (an exploitation arm) is played at time $t$, then no feedback is received (modeled by $Y_t(i) = 0$). 
If instead, action  $i \le K$ (an exploration arm) is played, then feedback $Y_t(i)$ is observed. This feedback structure describes an instance of online learning with feedback graphs, where the underlying graph is the one in \Cref{fig:feedback-graph} (left). 
The rewards incurred by the exploring arms are fixed and known irregardless of the action played, while the only way to learn the expected value of $\rho_t(i)$ for $i>K$ is to play the corresponding exploring action $i-K$.

    \subsection{Relating the two problems} 
    \label{proof:step3}
        In this final step of the proof, we formally map the repeated bilateral trade problem to the \gener{} one. This mapping is such that the adversary $\Pb^k$ for the former problem corresponds to the adversary $\Pb^k$ for the latter. 

    \paragraph{From prices to actions.}We start partitioning the upper triangle $\U$ in the following $2K$ disjoint regions:
\begin{itemize}[topsep=4pt,itemsep=0pt,leftmargin=9pt]
    \item $ \forall k \in [K-1], J_k = [v_k - \e, v_k + \e) \times [\frac{2}{3},\frac{5}{6}] $.
    \item $ J_K = [v_K - \e, v_K + \e] \times [\frac{2}{3},\frac{5}{6}] $
    \item $\forall k \in [K-1], J_{k+K} = \{ (p,q ) \in \U \mid v_k - \e \le p < v_k + \e \text{ and } q < \frac{2}{3}\}$
    \item $J_{2K} = \U \backslash \bigcup_{k=1}^{2K-1} J_k$
\end{itemize}
Note, the first $K$ partitions are in the exploring square $Q_6$, while the last $K$ contain a partition of the segment on the diagonal that contains the optimal price. To formalize this mapping, define $\iota \colon \U \to [2K]$ as the function that associates to each $(p,q) \in \U$ the unique $i \in [2K]$ such that $(p,q) \in J_{i}$ (\Cref{fig:feedback-graph}, right).

    \paragraph{Simulating the feedback.} Consider the \gener{} problem, its feedback is less rich than the one experienced in our original problem. However, we show how it is possible to simulate the latter via the former and some independent randomness. Formally, beside the sequence of $\{0,1\}^{2K}$-valued random variables $Y,Y_1,\dots, Y_T$, consider two sequences of $[0,1]$-valued random variables $U,U_1, \dots , U_T$ and $ V,V_1,\dots, V_T$ such that:
\begin{itemize}[topsep=4pt,itemsep=0pt,leftmargin=9pt]
    \item For each $k \in \{0,\dots,K\}$ the sequence $Y,Y_1,\dots, Y_T$ is $\Pb^k$-i.i.d. and follows the distribution in \Cref{eq:Yprob}
    \item For each $k \in \{0,\dots,K\}$ the sequence $V,V_1,\dots, V_T$ is $\Pb^k$-i.i.d. and $V$ is a uniform distribution under $\Pb^k$.
    \item For each $k \in \{0,\dots,K\}$, the following sequences are independent under $\Pb^k$: $(V,V_1,\dots V_T)$, $(U,U_1,\dots U_T)$, and $\lrb{ (S,B), (S_1,B_1),\dots, (S_T,B_T) }$.
\end{itemize}
    The random variables $U_t$ models the random bits of the learning algorithm, while $V_t$ are used to simulate, together with $Y_t$, the feedback $\brb{\I(S\le p), \I\{q \le B\}}$. Formally, we have this crucial claim.
\begin{restatable}{claim}{restatable}
\label{claim-5-appe}
For any $(p,q) \in \U$ there exists a function $\varphi_{p,q}:\{0,1\}\times [0,1]\to \{0,1\}^2$ such that, for all $k\in \{0,\dots,K\} $, the distributions under $\mathbb{P}^k$ of $\varphi_{p,q}(Y( \iota(p,q) ),V)$ and  $\brb{\I(S\le p), \I\{q \le B\}}$ coincide.
\end{restatable}

    For all $(p,q) \in \U$, fix a $\varphi_{p,q}$ as in \Cref{claim-5-appe}. We get back to our learning algorithm $\cA$ for the repeated bilateral trade problem and we simulate a run of it against an i.i.d. sequence $(S,B)$ issued by the adversary $\Pb^k$ using a run of the \gener{} against the adversary $\Pb^k$ (for the latter problem). Formally,  we construct an algorithm $\tilde \cA$ (based on $\cA$ and the sequence of random seeds $V_1, V_2, \dots$) to solve this new problem in the following way: for each time $t$ the following happens
\begin{enumerate}[topsep=4pt,itemsep=0pt,leftmargin=9pt]
    \item If $\cA$ selects a pair $(\tilde P_t, \tilde Q_t) \in \cU$, then $\tilde \cA$ plays action $\tilde I_t = \iota (\tilde P_t, \tilde Q_t) \in [2K]$.
    \item We feed $\cA$ with the feedback $\fhi_{\tilde P_t, \tilde Q_t} \brb{ Y_t( \tilde I _t ), V_t } \in \{0,1\}^2$.
\end{enumerate}
By induction on $t$, \Cref{claim-5-appe} implies that for all $k\in\{0,\dots,K\}$ and $t\in[T]$, we have\footnote{Given a random variable $X$ we denote with $\Pb_X$ the pushforward of the probability measure $\Pb$ to the codomain of $X$. Furthermore, we use $\bbL$ to denote the Lebesgue measure. We refer to \Cref{s:inverse-transform} for further details.}
\[
    \Pb^k_{ ( \tilde P_t , \tilde Q_t ) }
=
    \Pb^k_{ ( P_t , Q_t ) }.
\]
This means that the feedback received by $\cA$ under the simulation (i.e., using the $Y_t$ and the random seeds $V_t$) indeed behaves exactly like $\brb{\I(S\le p), \I\{q \le B\}}$ under $\Pb^k$. This observation, together with the fact that 
$
    \Pb^k_{(\tilde P_t, \tilde Q_t, Y_t)} 
= 
    \Pb^k_{(\tilde P_t, \tilde Q_t)}
    \otimes \Pb^k_{Y_t}
$ for all $k\in \{0,\dots,K\}$ and $t\in[T]$, yields
\begin{align*}
    R_T^k(\cA)
&
=
    T \bbE^k \bsb{ \gft(v_k, S, B) }
    - \sum_{t=1}^T \bbE^k\bsb{ \gft ( P_t, Q_t, S_t, B_t ) }
\\
&
\ge
    T \bbE^k \bsb{ \rho(k+K, Y) }
    - \sum_{t=1}^T \bbE^k \Bsb{ \rho \brb{ \iota(P_t,Q_t), Y_t } }
\\
&=
    T \bbE^k \bsb{ \rho(k+K, Y) }
    - \sum_{t=1}^T \bbE^k \Bsb{ \rho \brb{ \iota(\tilde P_t, \tilde Q_t), Y_t } }
\\
&
=
    T \bbE^k \bsb{ \rho(k+K, Y) }
    - \sum_{t=1}^T \bbE^k \bsb{ \rho ( \tilde I_t, Y_t ) }=
    \tilde R_T^k ( \tilde \cA ) \;,
\end{align*}
where $R_T^k(\cA)$ (resp., $\tilde R_T^k ( \tilde \cA )$) is the regret suffered by the algorithm $\cA$ (resp., $\tilde \cA$) after $T$ rounds of the bilateral trade problem with two-bit feedback (resp., the related problem on $2K$ actions) in the scenario $\Pb^k$. Note, the first inequality follows by definition of the reward $\rho$ and our choices of $\cplat$ and $\cspike.$
Summing over $k\in [K]$ and dividing by $K$, this implies
\[
    \frac1K \sum_{k\in[K]} R_T^k  (\cA)
\ge
    \frac1K \sum_{k\in[K]} \tilde R_T^k  (\tilde \cA)
\ge
    \inf_{\bar{\cA} \in \mathrm{Rand}} \frac1K \sum_{k\in[K]} \tilde R_T^k  (\bar{\cA})
=
    \inf_{\bar{\cA} \in \mathrm{Det}} \frac1K \sum_{k\in[K]} \tilde R_T^k  (\bar{\cA}) \;,
\]
where the first (resp., second) infimum is over the set $\mathrm{Rand}$ (resp., $\mathrm{Det}$) all randomized (resp., deterministic) algorithms $\bar{\cA}$ for the related problem on $2K$ actions, and the last standard equality is a straightforward consequence of the stochastic i.i.d.\ setting.

In \Cref{lemma:finale} in \Cref{app:lowerbound} we show that for any deterministic algorithm $\cA$ for the related problem on $2K$ actions, it either holds that $\frac1K \sum_{k\in[K]} \tilde R_T^k(\cA) \ge \frac{1}{50^3} T^{3/4}$ or that $\tilde R_T^0(\cA) \ge \frac{1}{50^3} T^{3/4}$. This, together with the inequalities above implies that there exists an $k\in \{0,\dots,K\}$ such that $R_T^k  (\cA) \ge \frac{1}{50^3} T^{3/4}$, concluding the proof of \Cref{thm:two-bit-two-prices-lower}.
\end{proof}

    \begin{algorithm*}[t]
        \begin{algorithmic}[t]
        \State \textbf{Input:} price $p$
        \State \textbf{Environment:} fixed pair of seller and buyer valuations $(s,b)$
        \State Toss a biased coin with probability $p$ of Heads
        \State \textbf{if} Heads \textbf{then} draw $U$ uniformly at random in $[0,p]$ and set $\hat p \gets U$, $\hat q \gets p$
        \State \textbf{else} draw $V$ uniformly at random in $[p,1]$ and set $\hat p \gets p$, $\hat q \gets V$
        \State Post price $\hat p$ to the seller and $\hat q$ to the buyer and observe the one-bit feedback $\ind{s \le \hat p\le \hat q \le b}$
        \State \textbf{Return} $\egft(p) \gets \ind{s \le \hat p\le \hat q \le b}$ \Comment{Unbiased estimator of $\gft(p)$}
         \caption*{\textbf{Estimation procedure of $\gft$ using two prices and one-bit feedback}}
    \end{algorithmic}
    \end{algorithm*}

\section{A \texorpdfstring{$T^{3/4}$}{Omega(T 3/4)} upper bound: one bit and two prices}
\label{sec:one-bit}

In this section, we introduce our algorithm, \blindExpThree{}, for the one-bit feedback setting against a $\sigma$-smooth adaptive adversary that achieves a bound on the regret of order $T^{3/4}$, up to logarithmic terms. 
A key technique that we use is a Monte Carlo estimation procedure $\egft$ (see pseudocode for details) that allows us to estimate the expected gain from trade $\bbE \bsb{ \gft(p, S_t, B_t) }$ of a price $p$, by posting {\em two} different prices $(\hat p, \hat q)$ and receiving {\em one} bit of feedback.

    \begin{lemma}[Lemma 1 of \citet{AzarFF22}]
    \label{lem:estimators}
        Fix any agents' valuations $(s,b) \in [0,1]^2$. For any price $p \in [0,1]$, it holds that $\egft(p)$ is an unbiased estimator of $\gft(p)$, i.e., $\E{\egft(p)} = \gft(p)$, where the expectation is with respect to the randomness of the estimation procedure.
    \end{lemma}

    Once we have this procedure, we can present our algorithm. 
    At a high level, it mimics the behavior of Exp3 on a fixed discretization of $K$ prices, but uses the estimation procedure to perform the uniform exploration step. 
    Our algorithm is ``blind'' because---unlike what happens in the bandit case---posting a price does not reveal  the corresponding gain from trade. 
    With a careful analysis, we show that the uniform exploration term is indeed enough to achieve the tight regret bound of order $\tilde{O}(T^{3/4})$. 
    (We recall that the $\sigma$-smoothness of the valuations distributions is crucial to ensure that the performance of the best fixed price in hindsight on a grid is ``close enough'' to the performance of the best fixed price overall.)
    
    \begin{algorithm}[t]
    \caption*{\textbf{Learning algorithm with 1-bit feedback and two prices:} \blindExpThree{}}
    \begin{algorithmic}[t]
    \State \textbf{input:} Learning rate $\eta>0$, exploration rate $\gamma \in (0,1)$, grid of prices $G$, with $|G| = K$
    \State \textbf{initialization:} Set $w_1(i)$ to $1$ for all $i \in [K]$ and $W_1 = K$
    \For{time $t=1,2,\ldots$}
        \State Let $\pi_t(i) = \frac{w_t(i)}{W_t}$ for all $i \in [K]$ 
        \State Toss a biased coin with probability $\gamma$ of Heads
        \If{Tails} \Comment{Exploitation step}
        \State Post price $p_t$ drawn according to distribution $\pi_t$ and set $\hat r_t(i) = 0$ for all $i \in [K]$
        \Else \Comment{Exploration step}
            \State Draw a price $g_{i}$ uniformly at random in $G$
            \State Use the estimation procedure on price $g_{i}$ and receive $\egft_t(g_i)$
            \State Set $\hat r_t(i) = \frac{K}{\gamma}\cdot \egft_t(g_i)$ and $\hat r_t(j) = 0$ for all $j \neq i$.
        \EndIf
        \State Let $ w_{t+1}(i) = w_t(i) \cdot \exp\brb{ \eta \hat r_t(i)}$ for all $i \in [K]$ \Comment{Update of the Exponential Weights}
        \State Let $W_{t+1} = \sum_{p_i \in G} w_{t+1}(i)$
        
    \EndFor
    \end{algorithmic}
    \end{algorithm}
    
    \begin{theorem}
    \label{thm:one-bit-two-prices-upper}
        Consider the problem of repeated bilateral trade against a $\sigma$-smooth adaptive adversary in the one-bit feedback model, for any $\sigma \in (0,1]$. If we run \blindExpThree{} with exploration rate $\gamma \in (0,1)$, learning rate $\eta >0$, and the uniform $K$-grid $G$ such that $\frac{2\eta K}{\gamma} \le 1$, then, for each time horizon $T \in \N$, we have that
        \[
        R_T(\text
{\blindExpThree{}}) \le \frac{\ln K}{\eta} + \brb{ \gamma + \eta\frac{K}{\gamma} + \frac{1}{\sigma K} } T.
        \]
        In particular, if we set the number of grid points $K = \lfl{  T^{1/4} }$, the exploration rate $\gamma = \frac{ (\ln T)^{1/3} }{ T^{1/4} }$, and the learning rate $\eta = \frac{1}{2} \frac{ (\ln T)^{2/3} }{ T^{3/4} }$, then $R_T(\text{\blindExpThree{}}) \le \lrb{ \frac2\sigma + 3 (\ln T)^{1/3} }\cdot T^{3/4} \;.$
    \end{theorem}

\begin{proof}
    The analysis of \blindExpThree{} needs to carefully take into account many sources of randomness: the internal randomness of the algorithm, of the estimation procedures and of the $\sigma$-smooth distributions of the adversary. Note, moreover, that the adversary is non-oblivious, so the choice of the distribution $(S_t,B_t)$ depends on all the realizations of the past randomization. 
    Fix any exploration rate $\gamma \in (0,1)$, learning rate $\eta >0$ and number of grid points $K \in \N$ such that $\frac{2\eta K}{\gamma} \le 1$. Fix also any time horizon $T \in \N$. In the following, we use the random variables $(P_t,Q_t)$ to denote the randomized prices posted by the algorithm at time $t$. 
    
    Fix any history of the algorithm (i.e. realization of all the randomness involved). We have the following:
    \begin{align}
    \nonumber
            \ln \lrb{ \frac{W_{T+1}}{W_1} }
&=
    \ln \lrb{ \prod_{t=1}^T \frac{W_{t+1}}{W_t} }
=
    \sum_{t=1}^T \ln \lrb{\frac{W_{t+1}}{W_t}}
=
    \sum_{t=1}^T \ln \lrb{ \sum_{i \in [K]} \pi_t(i) \exp\lrb{\eta \hat{r}_t(i)}}
\\ \nonumber
&\le
    \sum_{t=1}^T \ln \lrb{ 1 + \eta \sum_{i \in [K]} \pi_t(i) \hat{r}_t(i) + \eta^2 \sum_{i \in [K]} \pi_t(i) \brb{\hat{r}_t(i)}^2}
    \\
\nonumber
&\le
    \eta \sum_{t=1}^T \sum_{i \in [K]} \pi_t(i) \hat{r}_t(i) + \eta^2  \sum_{t=1}^T \sum_{i \in [K]} \pi_t(i) \brb{\hat{r}_t(i)}^2 \tag*{(using $\hat{r}_t(i) \le \frac{K}{\gamma}$)}
\\
\label{eq:pivot}
&\le
    \eta \sum_{t=1}^T \sum_{i \in [K]} \pi_t(i) \hat{r}_t(i)\left[1 + \eta\frac{K}{\gamma}\right].
    \end{align}
Crucially, we can use the standard exponential and logarithmic inequalities $\exp(x) \le 1 + x + x^2$ (valid whenever $x \le 1$), and $\ln(1+x) \le x$ (valid whenever $x > -1$) only because the particular choice of the parameters $(\frac{2\eta K}{\gamma} \le 1)$ implies that $\eta \hat{r}_t(i) \le 1$ and 
    \[
        \eta \sum_{i \in [K]} \pi_t(i) \hat{r}_t(i) + \eta^2 \sum_{i \in [K]} \pi_t(i) \brb{\hat{r}_t(i)}^2 \le 2\eta \sum_{i \in [K]} \pi_t(i) \hat{r}_t(i) \le \frac{K}{\gamma}.
    \]
   
   Inequality \ref{eq:pivot} is the pivot of our analysis, as we construct upper and lower bounds to its two extremes. We start from its first term, take the expectation with respect to the whole randomness of the process and consider any price $g_i$ in the grid $G$:
\begin{align}
\nonumber
     \E{\ln \lrb{ \frac{W_{T+1}}{W_1}}} &= \E{\ln \lrb{ W_{T+1}}} - \ln K \ge  \E{\ln \lrb{ w_{T+1}(i)}} - \ln K\\
\label{eq:lower}     &= \eta \sum_{t=1}^T \E{ \hat{r}_t(i)} - \ln K = \eta \sum_{t=1}^T \E{\gft_t(g_i)} - \ln K.
\end{align}
The only delicate passage of the previous formula is the last equality, where we used that $\E{\hat{r}_t(i)} = \E{\gft_t(g_i)}$. To see why the latter holds, consider the filtration $\{\cF_t\}_t$ relative to the story of the process: $\cF_t$ is the $\sigma$-algebra generated by all the random variables involved in the process up to time $t$ (excluded). Moreover, let $\cE^i_t$ be the event that at round $t$ the coin toss results in head and the price selected u.a.r. for exploration is $g_i$. We have the following:
\begin{align*}
    \E{\hat{r}_t(i) \mid \cF_t} &= \E{\ind{\cE_t^i}\hat{r}_t(i) \mid \cF_t} \tag*{$\hat{r}_t(i) = \ind{\cE_t^i}\hat{r}_t(i)$}\\
    &= \E{\ind{\cE_t^i} \E{\hat{r}_t(i) \mid \cF_t,\cE_t^i}\mid \cF_t}\tag*{Law of total exp.}\\
    &=\frac{K}{\gamma}\E{\ind{\cE_t^i} \E{\egft_t(g_i) \mid \cF_t,\cE_t^i}\mid \cF_t} \tag*{Def. of $\hat{r}_t(i)$}\\
    &= \frac{K}{\gamma}\Pb[\cE_t^i\mid \cF_t] \E{\gft_t(g_i) \mid \cF_t} \tag*{Lemma \ref{lem:estimators} and $(S_t,B_t)$ indep. of $\cE_t^i$}\\
    &= \E{\gft_t(g_i) \mid \cF_t} 
\end{align*}
For the final step, note that, conditioned on $\cF_t$, the event $\cE^i_t$ has probability $\frac{\gamma}K$: the random coin gives Tails with probability $\gamma$ and price $g_i$ is chosen (independently) u.a.r. as the one to be actually explored with probability $1/K$. Taking the expectation with respect to $\cF_t$ gives that $\E{\hat{r}_t(i)} = \E{\gft_t(g_i)}$.

Let's go back to \Cref{eq:pivot} and focus on the last term. Conditioning with respect to $\cF_t$:
\[
    \E{\pi_t(i) \hat{r}_t(i)\mid \cF_t} = \pi_t(i) \E{\hat{r}_t(i)\mid \cF_t} = \pi_t(i) \E{\gft_t(g_i)\mid \cF_t}.
\]
Taking the expectation with respect to $\cF_t$ and summing over all the $g_i \in G$, we have the following:
    \begin{align}
    \label{eq:upper}
    \E{\gft_t(P_t,Q_t)} \ge (1-\gamma)\sum_{i \in [K]}\E{\pi_t(i) \gft_t(g_i)}=(1-\gamma)\sum_{i \in [K]}\E{\pi_t(i) \hat{r}_t(i)},
    \end{align}
    where the first inequality follows from the fact that with probability $1-\gamma$ the learner at time $t$ chooses exploitation and thus posts a price in the grid $G$ according to distribution $\pi_t.$
    We can plug \Cref{eq:lower} and \Cref{eq:upper} into \Cref{eq:pivot} to obtain the following:
    \[
        \eta \sum_{t=1}^T \E{\gft_t(g_i)} - \ln K \le \frac{\eta}{1-\gamma}\left(1 + \eta\frac{K}{\gamma}\right)\sum_{t=1}^T \E{\gft_t(P_t,Q_t)} 
    \]
    Multiplying everything by $\frac{1-\gamma}\eta$, rearranging, and using that the gain from trade is always upper bounded by $1$, we get:
    \[
        \sum_{t=1}^T \E{\gft_t(g_i)} - \sum_{t=1}^T \E{\gft_t(P_t,Q_t)} \le \frac{\ln K}{\eta} + \left(\gamma + \eta\frac{K}{\gamma}\right)T
    \]
    The argument so far holds for any adaptive adversary $\cS$ and any choice of price on the grid $g_i$. This, together with the discretization result \Cref{claim:discretization} gives the desired bound:
    \[
        R_T(\text
{\blindExpThree{}}) \le \frac{\ln K}{\eta} + \left(\gamma + \eta\frac{K}{\gamma} + \frac{1}{\sigma K}\right)T \qedhere
    \]
\end{proof}

    We remark that we tune \blindExpThree{} without relying on $\sigma$, obtaining guarantees for a learner that is oblivious to the smoothness parameter.

\section{Conclusions and open problems}

In this paper, we initiated the study of $\sigma$-smooth adversaries in online learning for pricing problems. Focusing on the repeated bilateral trade problem, we proved that one bit of feedback is sufficient to achieve sublinear regret, pushing the boundary of learnability beyond the i.i.d.\ setting. We hope that the smoothed adversarial approach will find more applications to learning pricing strategies that cannot otherwise be efficiently learned in the adversarial model under partial feedback.  

The surprising minimax regret regime of $T^{3/4}$ surpasses the $\sqrt{T}$ vs. $T^{2/3}$ dichotomy observed in other partial feedback models (e.g., partial monitoring and feedback graph), and motivates the intriguing question of whether techniques based on the generalized information ratio \citep{lattimore2019information} could be used to define a unified algorithmic tool in our framework and, more generally, to analyze online problems in digital markets.

\bibliographystyle{plainnat}
\bibliography{references}

\clearpage
\appendix

\section{One-bit/two-scenarios inverse-transformation representability}
\label{s:inverse-transform}

We recall that given two probability measures $\Pb$ and $\bbQ$ on a measurable space $(\Omega, \cF)$, we say that $\bbQ$ is absolutely continuous with respect to $\Pb$ and we write $\bbQ \ll \Pb$ if for all $E\in \cF$ such that $\Pb[E]=0$, it holds that $\bbQ[E]=0$.
Moreover, if $\bbQ \ll \Pb$, the Radon-Nikodym theorem states that there exists a density (called Radon-Nikodym derivative of $\bbQ$ with respect to $\Pb$ and denoted by) $\frac{\dif \bbQ}{\dif\Pb}\colon \Omega \to [0,\iop)$ such that, for all $E\in \cF$, it holds that
\[
    \bbQ[E]
=
    \int_E \frac{\dif \bbQ}{\dif\Pb}(\omega) \diff \Pb(\omega) \;.
\]
For a reference of the previous result, see \cite[Theorem~13.4]{bass2013real}.

Moreover, if $(\Omega, \cF, \Pb)$ is a probability space, $(\cX, \cF_{\cX})$ is a measurable space, and $X$ is a random variable from $(\Omega, \cF)$ to $(\cX, \cF_{\cX})$, we denote by $\Pb_X$ the push-forward measure of $\Pb$ by $X$, i.e., the probability measure defined on $\cF_{\cX}$ by $\Pb_X[F] = \Pb[X \in F]$, for all $F \in \cF_{\cX}$.

If $(\Omega, \cF)$ and $(\Omega', \cF')$ are two measurable spaces, we denote by $\cF \otimes \cF'$ the $\sigma$-algebra of subsets of $\Omega \times \Omega'$ generated by the collection of subsets of the form $F\times F'$, where $F \in \cF$ and $F' \in \cF'$.
If $(\Omega, \cF, \Pb)$ and $(\Omega', \cF',\Pb')$ are two probability spaces, we denote the product measure of $\Pb$ and $\Pb'$ by $\Pb\otimes\Pb'$, i.e., $\Pb\otimes\Pb'$ is the unique probability measure defined on $\cF \otimes \cF'$ which satisfies $(\Pb\otimes\Pb')[F\times F'] = \Pb[F]\Pb'[F']$, for all $E\in \cF$ and $E'\in \cF'$.

If $(\Omega, \cF, \Pb)$ is a probability space, $(\cX, \cF_{\cX})$ and $(\cY, \cF_{\cY})$ are measurable spaces, $X$ is a random variable from  $(\Omega, \cF)$ to $(\cX, \cF_{\cX})$, and $Y$ is a random variable from $(\Omega, \cF)$ to $(\cY, \cF_{\cY})$, we denote the conditional probability of $X$ given $Y$ by $\Pb_{X \mid Y}$, i.e., $\Pb_{X \mid Y}[E] = \Pb[X \in E \mid Y]$, for each $E \in\cF_{\cX}$. In this case, for each $E \in \cF_{\cX}$, we recall that $\Pb_{X \mid Y}[E]$ is a $\sigma(Y)$-measurable random variable.
Furthermore, if $X'$ is another random variable from  $(\Omega, \cF)$ to some measurable space $(\cX', \cF_{\cX'})$, $f$ and $g$ are two real-valued bounded measurable functions (respectively from $(\cX \otimes \cY , \cF_{\cX} \otimes \cF_{\cY})$ to the reals and from $(\cX' \otimes \cY , \cF_{\cX'} \otimes \cF_{\cY})$ to the reals), and both $(\cX, \cF_{\cX})$ and $(\cX, \cF_{\cX'})$ are measurable spaces that arise from considering the Borel subsets of separable and complete metric space $(\cX,d)$ and $(\cX',d')$ respectively, it holds that
\[
    \bbE \bsb{ f(X,Y) g(X',Y) \mid Y }
=
    \bbE \bsb{ f(X,Y) \mid Y } \cdot \bbE \bsb{  g(X',Y) \mid Y }
\]
whenever
\[
    \Pb_{(X,X') \mid Y}
=
    \Pb_{X \mid Y}\otimes\Pb_{X' \mid Y} \;.
\]

\subsection{Our inverse-transformation result}
\label{appe:inverse-transform-section}

In this section, we present a theorem that extends, in spirit, the classic inverse transformation method.
This result that can be of independent interest for replacing a type of feedback with another of better quality in lower-bound constructions based on reductions to simpler games.

\begin{definition}[Inverse-transformation representability]
    Let $(\Omega, \cF, \Pb)$ be a probability space and $\cB$ be the Borel $\sigma$-algebra of $[0,1]$.
    We say that $\Pb$ is \emph{inverse-transformation-representable} if there exists a measurable function $\psi$ from $\brb{[0,1], \cB}$  to $(\Omega, \cF)$ such that\footnote{We recall that $\leb$ is the Lebesgue measure on $\cB$.} $\Pb = \bbL_\psi$.
\end{definition}

The following theorem is a simple consequence of \cite[Corollary A.11]{takesaki1979operator}, and shows ``inverse-transformation representability in separable and complete metric spaces''.

\begin{theorem}
\label{t:inverse-transformation-method}
Suppose that $(\cY,d)$ is a separable and complete metric space, with $\cF_\cY$ as the Borel $\sigma$-algebra of $(\cY,d)$. Then any probability measure defined on $\cF_\cY$ is inverse-transformation-representable.
\end{theorem}

We are now ready to state the main theorem of this section. 
When we are uncertain about the underlying probability according to which some samples are drawn, and the uncertainty is between two probability measure $\Pb$ and $\bbQ$, the theorem provides a characterization under which we can simulate a random variable $Y$ using some independent random seed $U$ and having access to a 1-bit random variable $X$. 
This theorem can be of independent interest as a tool for lower bound reductions in online learning problems, as we used for example in \Cref{thm:two-bit-two-prices-lower}.
It establishes ``One-bit/two-scenarios inverse-transformation representability in separable and complete metric spaces''.

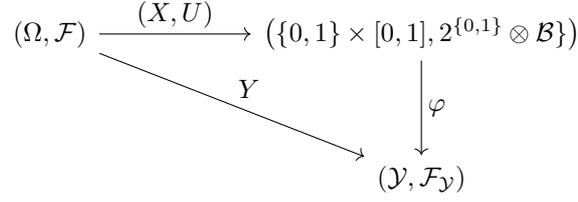
\begin{figure}
    \centering%
    \begin{tikzcd}
        {(\Omega,\mathcal F)} & {} & {\left(\{0,1\}\times[0,1],2^{\{0,1\}}\otimes\mathcal B\}\right)} \\
        \\
        && {(\mathcal Y, \mathcal F_{\mathcal Y})}
        \arrow["{{\varphi}}", from=1-3, to=3-3]
        \arrow["{(X,U)}", from=1-1, to=1-3]
        \arrow["{Y}", from=1-1, to=3-3]
    \end{tikzcd}
    \caption{\footnotesize{Pictorial representation of \Cref{t:inverse-transformation-method-2}. The way to interpret it is not event by event but in probability: the probability of a measurable set in $\mathcal F_Y$ can be computed in $\Omega$ equivalently via the pullback of $Y$, or of $\varphi\circ (X,U) $.}}
    \label{fig:diagram}
\end{figure}

\begin{theorem}
\label{t:inverse-transformation-method-2}
Suppose that $(\cY,d)$ is a separable and complete metric space with $\cF_\cY$ as the Borel $\sigma$-algebra of $(\cY,d)$. Let $(\Omega, \cF)$ be a measurable space, $X$ a random variable from $(\Omega, \cF)$ to $\brb{ \{0,1\}, 2^{\{0,1\}} }$,
$Y$ a random variable from $(\Omega,\cF)$ to $(\cY, \cF_\cY)$,
and $U$ random variable from $(\Omega, \cF)$ to $\brb{ [0,1], \cB }$, 
where $\cB$ is the Borel $\sigma$-algebra of $[0,1]$.
Suppose that $\Pb,\bbQ$ are probability measures defined on $\cF$, and $p \in (0,1)$, $q\in [0,1]$ are such that:
\begin{itemize}[topsep=4pt,itemsep=0pt,leftmargin=9pt]
    \item $\Pb[X=1] = p$ and $\bbQ[X=1] = q$.
    \item $U$ is a uniform random variable on $[0,1]$ both under $\Pb$ and $\bbQ$, i.e., we have that $\Pb_U = \leb = \bbQ_U$.
    \item $U$ is independent of $X$ both under $\Pb$ and $\bbQ$, i.e., $\Pb_{(X,U)} = \Pb_X \otimes \Pb_U$ and $\bbQ_{(X,U)} = \bbQ_X \otimes \bbQ_U$.
\end{itemize}
Then, the following are equivalent:
\begin{enumerate}
    \item\label{item:inverse-transform-one} There exists a measurable function $\fhi$ from $\brb{ \{0,1\}\times[0,1], 2^{ \{0,1\} } \otimes \cB }$ to $(\cY, \cF_\cY)$ such that
    \[
        \Pb_Y = \Pb_{\fhi(X,U)}
        \qquad \text{ and } \qquad
        \bbQ_Y = \bbQ_{\fhi(X,U)} \;.
    \]
    \item \label{item:inverse-transform-two} $\bbQ_Y \ll \Pb_Y$, and $\Pb_Y$-almost-surely it holds that
    \[
        \min \frac{\dif \bbQ_X}{\dif \Pb_X}
    \le
        \frac{\dif \bbQ_Y}{\dif \Pb_Y}
    \le
        \max \frac{\dif \bbQ_X}{\dif \Pb_X} \;.
    \] 
\end{enumerate}
\end{theorem}

\begin{proof}
We divide the proof in two parts, depending on whether or not $p=q$.

Assume first that $p\neq q$. 
In this case, we will prove the chain of equivalencies
\[
    \text{\Cref{item:inverse-transform-one} } \ \Leftrightarrow \
    \text{ \Cref{i:a} } \ \Leftrightarrow \ 
    \text{ \Cref{i:b} } \ \Leftrightarrow \
    \text{ \Cref{i:c} } \ \Leftrightarrow \ 
    \text{ \Cref{item:inverse-transform-two} } \;,
\]
where \Cref{i:a}, \Cref{i:b}, and \Cref{i:c} are the following propositions:
{
\setlist[enumerate]{label*=\alph*),ref=\alph*}
\begin{enumerate}
    \item \label{i:a} There exists two probability measures $\mu_0$ and $\mu_1$ over $\cF_\cY$ such that
    \[
        \Pb_Y = (1-p) \mu_0 + p \mu_1
    \qquad \text{ and } \qquad
        \bbQ_Y = (1-q) \mu_0 + q \mu_1 \;.
    \]
    \item \label{i:b} $\frac{q}{q-p} \Pb_Y - \frac{p}{q-p} \bbQ_Y \ge 0$ 
    and $\frac{1-p}{q-p} \Pb_Y - \frac{1-q}{q-p} \bbQ_Y \ge 0$ .
    \item \label{i:c} $\bbQ_Y \ll \Pb_Y$ and
    $
        \min \brb{ \frac{q}{p}, \, \frac{1-q}{1-p} }
    \le
        \frac{\bbQ_Y [A] }{\Pb_Y[A]}
    \le
        \max \brb{ \frac{q}{p}, \, \frac{1-q}{1-p} }
    $
    for all $A \in \cF_\cY$ such that $\Pb_Y[A] > 0$.
\end{enumerate}
}

We begin by proving that \Cref{item:inverse-transform-one} is equivalent to \Cref{i:a}.
Assume \Cref{item:inverse-transform-one}. 
Define $\mu_0 = \Pb_{\fhi(0,U)}$ and $\mu_1 = \Pb_{\fhi(1,U)}$.
Since $U$ is uniform under both under $\Pb$ and $\bbQ$, it also holds that $\mu_0 = \bbQ_{\fhi(0,U)}$ and $\mu_1 = \bbQ_{\fhi(1,U)}$. Thus
\begin{align*}
    \Pb_Y
&
    =
    \Pb_{\fhi(X,U)}
    =
    (1-p)\Pb_{\fhi(0,U)} + p\Pb_{\fhi(1,U)}
    =
    (1-p)\mu_0 + p\mu_1
\\
    \bbQ_Y
&
    =
    \bbQ_{\fhi(X,U)}
    =
    (1-q)\bbQ_{\fhi(0,U)} + q\bbQ_{\fhi(1,U)}
    =
    (1-q)\mu_0 + q\mu_1 \;,
\end{align*}
where we used that fact that $X$ and $U$ are independent both under $\Pb$ and $\bbQ$ and that $\Pb[X=1]=p$, $\bbQ[X=1]=q$.
This proves \Cref{i:a}.

Vice versa, assume \Cref{i:a}.
By \Cref{t:inverse-transformation-method}, we can find two measurable functions $\psi_0, \psi_1$ from $\brb{ [0,1], \cB }$ to $(\cY, \cF_\cY)$ such that $\mu_0 = \leb_{\psi_0}$ and $\mu_1 = \leb_{\psi_1}$ and define 
\[
    \fhi(x,u) 
= 
    \begin{cases}
        \psi_0 (u) & \text{ if } x = 0 \\
        \psi_1 (u) & \text{ if } x = 1
    \end{cases}
\]
for all $x\in\{0,1\}$ and $u\in[0,1]$.
Then $\fhi$ is a measurable function from $\brb{ \{0,1\}\times[0,1], 2^{ \{0,1\} } \otimes \cB }$ to $(\cY, \cF_\cY)$, and since $X$ is independent of $U$ and $U$ is uniform on $[0,1]$ both under $\Pb$ and $\bbQ$, we have
\begin{align*}
    \Pb_{\fhi(X,U)}
&
    =
    (1-p)\Pb_{\fhi(0,U)} + p\Pb_{\fhi(1,U)}
    =
    (1-p)\Pb_{\psi_0(U)} + p\Pb_{\psi_1(U)}
\\
&
    =
    (1-p)\leb_{\psi_0} + p\leb_{\psi_1}
    =
    (1-p)\mu_0 + p\mu_1
    =
    \Pb_Y
\\
    \bbQ_{\fhi(X,U)}
&
    =
    (1-q)\bbQ_{\fhi(0,U)} + q\bbQ_{\fhi(1,U)}
    =
    (1-q)\bbQ_{\psi_0(U)} + q\bbQ_{\psi_1(U)}
\\
&
    =
    (1-q)\leb_{\psi_0} + q\leb_{\psi_1}
    =
    (1-q)\mu_0 + q\mu_1
    =
    \bbQ_Y
\end{align*}
This proves \Cref{item:inverse-transform-one} and in turn yields that \Cref{item:inverse-transform-one} is equivalent to \Cref{i:a}.

We now prove that \Cref{i:a} is equivalent to \Cref{i:b}.
Assume \Cref{i:a}.
Then, for each $A \in \cF_{\cY}$ we have that the pair $\brb{\mu_0[A],\mu_1[A]}$ is the (only) solution of the linear system
\[
    \begin{cases}
        (1-p) x_0 + p x_1 &= \Pb_Y[A]  \\
        (1-q) x_0 + q x_1 &= \bbQ_Y[A]
    \end{cases}
\]
in the two variables $(x_0,x_1)$, which implies
\[
    \mu_0[A] = \frac{q}{q-p} \Pb_Y[A] - \frac{p}{q-p} \bbQ_Y[A]
\qquad \text{ and } \qquad
    \mu_1[A] =  \frac{1-p}{q-p} \bbQ_Y[A] - \frac{1-q}{q-p}\Pb_Y[A] \;.
\]
Since $\mu_0$ and $\mu_1$ are (non-negative) measures, this implies \Cref{i:b}.

Vice versa, assume \Cref{i:b}.
Define
\[
    \mu_0 = \frac{q}{q-p} \Pb_Y - \frac{p}{q-p} \bbQ_Y
\qquad \text{ and } \qquad
    \mu_1 =  \frac{1-p}{q-p} \bbQ_Y - \frac{1-q}{q-p}\Pb_Y \;.
\]
Since $\mu_0$ and $\mu_1$ are a linear combination of measures, they are signed measures and, by \Cref{i:b}, actually, they are (non-negative) measures. The fact that they are also probability measures follows trivially from $\Pb_Y[\cY] = 1 = \bbQ_Y[\cY]$. Now, a direct verification shows that $\Pb_Y = (1-p) \mu_0 + p \mu_1$ and $\bbQ_Y = (1-q) \mu_0 + q \mu_1$, i.e., that \Cref{i:a} holds.
We have then proved that \Cref{i:a} is equivalent to \Cref{i:b}.

We now prove that \Cref{i:b} is equivalent to \Cref{i:c}.
Firstly, note that by elementary linear-algebra (dividing by $\tilde p$ and solving by $\tilde q /\tilde p$ the linear system of inequalities), for each $\tilde q \in [0,1]$ and $\tilde p \in (0,1]$, the following equivalence holds
\begin{equation}
    \label{e:linear-algebra}
    \begin{cases}
    \displaystyle{ \frac{q}{q-p} \tilde p - \frac{p}{q-p} \tilde q \ge 0 }
    \\[2ex]
    \displaystyle{ \frac{1-p}{q-p} \tilde q - \frac{1-q}{q-p} \tilde p \ge 0 }
    \end{cases}
    \iff \ \
    \min \lrb{ \frac{q}{p}, \, \frac{1-q}{1-p} }
    \le
    \frac{\tilde q}{\tilde p}
    \le
    \max \lrb{ \frac{q}{p}, \, \frac{1-q}{1-p} \textbf{}}
\end{equation}

Assume \Cref{i:b}.
Note that if $p<q$ (resp., $q<p$), then if $A\in \cF_\cY$ is such that $\Pb_Y[A]=0$, the first (resp., second) inequality in \Cref{i:b} implies that also $\bbQ_Y[A]=0$, which in turn yields $\bbQ_Y \ll \Pb_Y$.
Furthermore, for each $A\in \cF_\cY$ such that $\Pb_Y[A] \neq 0$, the equivalence in \eqref{e:linear-algebra} with $\tilde p = \Pb_Y[A]$ and $\tilde q = \bbQ_Y[A]$ implies that 
\[
    \min \lrb{ \frac{q}{p}, \, \frac{1-q}{1-p} }
    \le
    \frac{\bbQ_Y[A]}{\Pb_Y[A]}
    \le
    \max \lrb{ \frac{q}{p}, \, \frac{1-q}{1-p} \textbf{}}
\]
which yields \Cref{i:c}.

Vice versa, assume \Cref{i:c}.
Note that \Cref{i:b} holds 
\begin{itemize}[topsep=4pt,itemsep=0pt,leftmargin=9pt]
    \item For all $A \in \cF_\cY$ such that $\Pb_Y[A] = 0$, because in this case also $\bbQ_Y[A] = 0$
    \item For all $A \in \cF_\cY$ such that $\Pb_Y[A] \neq 0$, by the equivalence in \eqref{e:linear-algebra} with $\tilde p = \Pb_Y[A]$ and $\tilde q = \bbQ_Y[A]$
\end{itemize}  
This proves that \Cref{i:b} and \Cref{i:c} are equivalent.

We now prove that \Cref{i:c} is equivalent to \Cref{item:inverse-transform-two}.
Assume \Cref{i:c}. Assume by contradiction that \Cref{item:inverse-transform-two} does not hold. Then, there exists $A \in \cF_{\cY}$ such that $\Pb_Y[A] > 0$ such that either for all $y \in A$ it holds that $\max\brb{\frac{\dif \bbQ_X}{\dif \Pb_X}} < \frac{\dif \bbQ_Y}{\dif \Pb_Y}(y)$ or it holds that $\min\brb{\frac{\dif \bbQ_X}{\dif \Pb_X}} > \frac{\dif \bbQ_Y}{\dif \Pb_Y}(y)$. In the first case
\[
    \max\lrb{\frac{\dif \bbQ_X}{\dif \Pb_X}}
=
    \max \lrb{ \frac{q}{p}, \, \frac{1-q}{1-p} }
\ge
    \frac{\bbQ_Y[A]}{\Pb_Y[A]}
=
    \frac{1}{\Pb_Y[A]} \int_A \frac{\dif \bbQ_Y}{\dif \Pb_Y} \diff \Pb_Y
>
    \max\lrb{\frac{\dif \bbQ_X}{\dif \Pb_X}} \;,
\]
yielding the contradiction we were seeking. The second case yields a contradiction in an analogous manner.

Vice versa, assume \Cref{item:inverse-transform-two}.
Then, if $A \in \cF_{\cY}$ is such that $\Pb_Y[A] > 0$, notice that
\[
    \min \lrb{ \frac{q}{p}, \frac{1-q}{1-p} }
=
    \min\lrb{\frac{\dif \bbQ_X}{\dif \Pb_X}}
\le
    \frac{1}{\Pb_Y[A]} \int_A \frac{\dif \bbQ_Y}{\dif \Pb_Y} \diff \Pb_Y
\le
    \max\lrb{\frac{\dif \bbQ_X}{\dif \Pb_X}}
=
    \max \lrb{ \frac{q}{p}, \frac{1-q}{1-p} } 
\]
which together with
\[
    \frac{\bbQ_Y[A]}{\Pb_Y[A]}
=
    \frac{1}{\Pb_Y[A]} \int_A \frac{\dif \bbQ_Y}{\dif \Pb_Y} \diff \Pb_Y
\]
(since $\bbQ_Y \ll \Pb_Y$), implies \Cref{i:c}.
This proves that \Cref{i:c} and \Cref{item:inverse-transform-two} are equivalent and shows in turn that \Cref{item:inverse-transform-one} is equivalent to \Cref{item:inverse-transform-two} whenever $p\neq q$.

Assume now that $p=q$.
Assume \Cref{item:inverse-transform-one}.
Since $X$ is independent of $U$ and $U$ is uniform on $[0,1]$ both under $\Pb$ and $\bbQ$, we get
\[
    \Pb_Y
=
    \Pb_{\fhi(X,U)}
=
    (1-p)\Pb_{\fhi(0,U)} + p\Pb_{\fhi(1,U)}
=
    (1-q)\bbQ_{\fhi(0,U)} + q\bbQ_{\fhi(1,U)}
=
    \bbQ_{\fhi(X,U)}
=
    \bbQ_Y \;.
\]
Hence, in particular $\bbQ_Y \ll \Pb_Y$ and $\frac{\dif \bbQ_Y}{\dif \Pb_Y} = 1$ $\Pb_Y$-almost-surely, which, together with the fact
\[
    \min\lrb{ \frac{\dif \bbQ_X}{\dif \Pb_X} }
\le
    1
\le
    \max\lrb{ \frac{\dif \bbQ_X}{\dif \Pb_X} }
\]
implies \Cref{item:inverse-transform-two}.

Vice versa, assume \Cref{item:inverse-transform-two}.
Fix a measurable function $\psi$ from $\brb{ [0,1], \cB }$ to $(\cY, \cF_\cY)$ such that $\Pb_Y = \leb_\psi$ (whose existence is guaranteed by \Cref{t:inverse-transformation-method}). 
Let $\fhi(x,u)=\psi(u)$ for all $x\in \{0,1\}$ and $u\in [0,1]$.
Being $U$ uniform both under $\Pb$ and $\bbQ$, we get that $\Pb_{\fhi(X,U)} = \Pb_{\psi(U)}= \leb_\psi = \bbQ_{\psi(U)} = \bbQ_{\fhi(X,U)}$. 
Moreover, since $p=q$, we have that $\min \frac{\dif \bbQ_X}{\dif \Pb_X} = 1 = \max \frac{\dif \bbQ_X}{\dif \Pb_X}$, which, together with \Cref{item:inverse-transform-two}, yields that, for any $A \in \cF_\cY$, 
\[
    \bbQ_Y[A] 
= 
    \int_A \frac{\dif \bbQ_Y}{\dif \Pb_Y} \diff \Pb_Y 
= 
    \int_A 1 \diff \Pb_Y
= 
    \Pb_Y[A] \;,
\]
thus $\Pb_Y = \bbQ_Y$.
Putting everything together, since we proved that all distributions $\Pb_{\fhi(X,U)}$, $\Pb_Y$, $\bbQ_{\fhi(X,U)},$ $\bbQ_Y$ are equal to each other, we obtain \Cref{item:inverse-transform-one}, concluding the proof.
\end{proof}

\section{Missing proofs from Section \ref{sec:lowerbound}}
\label{app:lowerbound}
\restatable*
\begin{proof}
A direct verification shows that, for all $(p,q)\in Q_6$ and $k\in[K]$, it holds that
\[
    \min \lrb{ \frac{\dif \Pb^k_{Y(k)}}{\dif \Pb^0_{Y(k)}} }
=
    1-2\cprob \cdot \e
\le
    \frac{\dif \Pb^k_{\lrb{ \I(S\le p), \I\{q \le B\}}}}{\dif \Pb^0_{\lrb{\I(S\le p), \I\{q \le B\}}}}
\le
    1+2\cprob \cdot \e
=
    \max \lrb{ \frac{\dif \Pb^k_{Y(k)}}{\dif \Pb^0_{Y(k)}} }
\]
and $\Pb^k_{\lrb{ \I(S\le p), \I\{q \le B\}}} \ll \Pb^0_{\lrb{\I(S\le p), \I\{q \le B\}}}$. 
For each $(p,q) \in Q_6$, by \Cref{t:inverse-transformation-method-2}, there exists (and we fix)
\[
    \fhi_{p,q} \colon \{0,1\} \times [0,1] \to \{0,1\}^2
\]
such that
\[
    \Pb^{\iota(p,q)}_{\fhi_{p,q}(Y(\iota(p,q)),V)}
=
    \Pb^{\iota(p,q)}_{\lrb{ \I(S\le p), \I\{q \le B\}}}
\qquad \text{ and } \qquad
    \Pb^{0}_{\fhi_{p,q}(Y(\iota(p,q)),V)}
=
    \Pb^{0}_{\lrb{ \I(S\le p), \I\{q \le B\}}} \;.
\]
Since for all $(p,q)\in Q_6$ and all $k \in [K]\m\bcb{ \iota(p,q) }$, we have  
$
    \Pb^{k}_{\lrb{ \I(S\le p), \I\{q \le B\}}}
=
    \Pb^{0}_{\lrb{ \I(S\le p), \I\{q \le B\}}}
$ (by \Cref{cl:perturbation}) and
$
    \Pb^{k}_{\fhi_{p,q}(Y(\iota(p,q)),V)}
=
    \Pb^{0}_{\fhi_{p,q}(Y(\iota(p,q)),V)}
$, then, for all $(p,q)\in Q_6$ and all $k\in\{0,\dots,K\}$, it holds that
\[
    \Pb^{k}_{\fhi_{p,q}(Y(\iota(p,q)),V)}
=
    \Pb^{k}_{\lrb{ \I(S\le p), \I\{q \le B\}}} \;.
\]
Moreover, since for all $(p,q)\in \cU \m Q_6$ and for all $k\in\{0,\dots,K\}$, it holds that 
$
    \Pb^{k}_{\lrb{ \I(S\le p), \I\{q \le B\}}} 
= 
    \Pb^{0}_{\lrb{ \I(S\le p), \I\{q \le B\}}}
$ (by \Cref{cl:perturbation}), then, by \Cref{t:inverse-transformation-method}, there exists (and we fix) 
\[
    \tilde \fhi_{p,q} \colon [0,1] \to \{0,1\}^2
\]
such that, for all $k\in\{0,\dots,K\}$, it holds that
\[
    \Pb^{k}_{\tilde \fhi_{p,q}(V)}
=
    \Pb^{k}_{\lrb{ \I(S\le p), \I\{q \le B\}}} \;.
\]
Defining for all $(p,q) \in \cU \m Q_6$ and $(y,v) \in \{0,1\} \times [0,1] $, $\fhi_{p,q}(y,v) = \tilde \fhi_{p,q}(v)$, we obtain the result.
\end{proof}

\begin{lemma}
\label{lemma:finale}
    Fix any deterministic algorithm $\cA$ for the related \gener{} problem on $2K$ actions, then at least one of the following two inequality holds: $\frac1K \sum_{k\in[K]} \tilde R_T^k  (\cA) \ge \frac{1}{50^3} T^{3/4}$ or $\tilde R_T^0(\cA) \ge \frac{1}{50^3} T^{3/4}$.
\end{lemma}
\begin{proof}
For any deterministic algorithm $\cA$ for the related problem on $2K$ actions, let $I_1, I_2, \dots$ be the actions played by $\cA$ on the basis of the sequential feedback received $Z_1,Z_2, \dots$ and define $N_t(i)$ and $M_t(i)$ as the random variables counting the number of times the learning algorithm $\cA$ plays action $i$, respectively $i+K$, up to time $t$, for any $i \in [K]$ and any time $t \in [T]$:
\[
    N_t(i) = \sum_{s=1}^t \I\{I_s = i\}, \,\, M_t(i) = \sum_{s=1}^t \I\{I_s = i + K\}.
\]
Using these variables, we can define the $N_t$ and $M_t$ as the counters of how many times exploiting, respectively exploring, actions have been played up to time $t$, for any $t \in [T]$:
\[
    N_t = \sum_{i\in[K]} N_t(i) \;, \
    M_t = \sum_{i\in[K]} M_t(i) \;.
\]
We relate the expected values of $M_T(k)$ under $\Pb^0$ and $\Pb^k$ as a function of the expected number of times the algorithm plays the corresponding exploring actions ($N_T(k)$). This formalizes the intuition that to discriminate between the different $\Pb^k$ the learner needs to play exploring actions. 
\begin{claim}
\label{claim:finale}
    The following inequality holds true for any $k\in[K]$:
    \begin{equation}
        \bbE^k \bsb{ M_T(k) } - \bbE^0 \bsb{ M_T(k) } \le  \cprob \cdot \e \cdot T \cdot \sqrt{ 2 \bbE^0 [N_T(k)] }.    
    \end{equation}
\end{claim}
\begin{proof}[Proof of \Cref{claim:finale}]
    For any $t\in [T]$, the action $I_t = I_t(Z_1, \dots, Z_{t-1})$ selected by $\cA$ at round $t$ is a deterministic function of $Z_1, \dots, Z_{t-1}$, for each $k\in[K]$. In formula, we then have the following
\begin{align}
\nonumber
 \bbE^k \bsb{ M_T(k) } - \bbE^0 \bsb{ M_T(k) } &=
    \sum_{t = 2}^T \Brb{ \Pb^k\bsb{ I_t (Z_1, \dots, Z_{t-1} ) = k + K } - \Pb^0\bsb{ I_t (Z_1, \dots, Z_{t-1} ) = k + K } }\\
\label{eq:TV}&\le
    \sum_{t = 2}^T \bno{ \Pb^k_{(Z_1, \dots, Z_{t-1} )} - \Pb_{(Z_1, \dots, Z_{t-1} )}^0 }_{\mathrm{TV}},
\end{align}
where $\lno{\cdot}_{\mathrm{TV}}$ denotes the total variation norm. We move now our attention towards bounding the total variation norm. To that end we use Pinsker's inequality and apply the chain rule for the KL divergence $\kl$. For each $k \in [K]$ and $t \in [T]$ we have the following:

\begin{align}
\nonumber
    \bno{ \Pb_{(Z_1,\dots,Z_t)}^0 &- \Pb^k_{(Z_1,\dots,Z_t)}}_{\mathrm{TV}}
\le 
    \sqrt{\frac12 \kl \brb{ \Pb_{(Z_1,\dots,Z_t)}^0,\, \Pb^k_{(Z_1,\dots,Z_t)} }}
\\
\label{eq:secondKL}
&\le
    \sqrt{ \frac12 \lrb{ \kl\brb{ \Pb_{Z_1}^0, \, \Pb^k_{Z_1} } + \sum_{s=2}^t \bbE\Bsb{ \kl \brb{ \Pb_{Z_s \mid Z_1,\dots,Z_{s-1} }^0 , \, \Pb^k_{Z_s \mid Z_1,\dots,Z_{s-1}} } } } }
\end{align}
We bound the two KL terms separately. $\cA$ is a deterministic algorithm, thus $I_1$ is a fixed element of $[2K]$, which implies that, for all $k\in [K]$,
\begin{align}
\nonumber
&
    \kl\brb{ \Pb_{Z_1}^0, \, \Pb^k_{Z_1} }
\\
\nonumber 
& \quad
=
    \lrb{ 
    \ln\lrb{ \frac{\Pb^0[Y_1(k) = 0]}{ \Pb^k[Y_1(k) = 0] } } \Pb^0[Y_1(k) = 0]
    +
    \ln\lrb{ \frac{\Pb^0[Y_1(k) = 1]}{ \Pb^k[Y_1(k) = 1] } } \Pb^0[Y_1(k) = 1]
    } \I\bcb{ I_1 = k }
\\
\label{eq:KLterm1}
& \quad
=
    \frac12\lrb{
    \ln \frac{\nicefrac{1}{2}}{\nicefrac{1}{2} - \cprob \cdot \e}
    +
    \ln\frac{\nicefrac{1}{2}}{\nicefrac{1}{2} + \cprob \cdot \e}
    } \cdot \I\{I_1 = k\}
\end{align}
Similarly, since $\cA$ is a deterministic algorithm, for all $s \ge 2$, the action $I_s = I_s(Z_1,\dots,Z_{s-1})$ selected by $\cA$ at time $t$ is a function of $Z_1,\dots, Z_{s-1}$ only, which implies, for all $k \in [K]$, 
\begin{align}
\nonumber
&
    \kl \brb{ \Pb_{Z_s \mid Z_1,\dots,Z_{s-1} }^0 , \, \Pb^k_{Z_s \mid Z_1,\dots,Z_{s-1}} }
\\
\nonumber
& 
\qquad =
    \bbE^0 \left[
        \ln \lrb{ \frac{ \Pb^0[Z_s = 0 \mid Z_1, \dots, Z_{s-1} ] }{ \Pb^k[Z_s = 0 \mid Z_1, \dots, Z_{s-1} ] } } \Pb^0[Z_s = 0 \mid Z_1, \dots, Z_{s-1} ]
        \right.
\\
\nonumber
& \qquad 
        \qquad
        +
        \left.
        \ln \lrb{ \frac{ \Pb^0[Z_s = 1 \mid Z_1, \dots, Z_{s-1} ] }{ \Pb^k[Z_s = 1 \mid Z_1, \dots, Z_{s-1} ] } } \Pb^0[Z_s = 1 \mid Z_1, \dots, Z_{s-1} ]
    \right]
\\
\nonumber
&
\qquad
=
    \bbE^0 \bigg[ \lrb{ 
    \ln\lrb{ \frac{\Pb^0[Y_s(k) = 0]}{ \Pb^k[Y_s(k) = 0] } } \Pb^0[Y_s(k) = 0]
    +
    \ln\lrb{ \frac{\Pb^0[Y_s(k) = 1]}{ \Pb^k[Y_s(k) = 1] } } \Pb^0[Y_s(k) = 1]
    } 
\\
\nonumber& \qquad \qquad 
    \times \I\bcb{ I_s (Z_1, \dots, Z_{s-1}) = k } \bigg]
\\
\label{eq:KLterm2}
&
\qquad
=
    \frac12\lrb{
    \ln \frac{\nicefrac{1}{2}}{\nicefrac{1}{2} - \cprob \cdot \e}
    +
    \ln\frac{\nicefrac{1}{2}}{\nicefrac{1}{2} + \cprob \cdot \e}
    }
    \Pb^0 \bsb{ I_s (Z_1, \dots, Z_{s-1}) = k } 
\end{align}
For $T\ge 8008$ the following useful inequality holds:
\begin{equation}
\label{eq:useful}
    \frac12\lrb{
    \ln \frac{\nicefrac{1}{2}}{\nicefrac{1}{2} - \cprob \cdot \e}
    +
    \ln\frac{\nicefrac{1}{2}}{\nicefrac{1}{2} + \cprob \cdot \e}
    }
\le
    4 \cdot \cprob^2 \cdot \e^2.
\end{equation}
We can combine the inequalities in \Cref{eq:KLterm1} and \Cref{eq:KLterm2} into \Cref{eq:secondKL} and plug in the bound in to obtain:
\[
    \lno{ \Pb_{(Z_1,\dots,Z_t)}^0 - \Pb^k_{(Z_1,\dots,Z_t)}}_{\mathrm{TV}} \le \cprob \cdot \e \cdot \sqrt{ 2 \bbE [N_t(k)] }
\]
Once we have this upper bound on the total variations of the random variables $(Z_1,\dots,Z_t)$ under $\Pb^0$ and $\Pb^k$ we can get back to the initial \Cref{eq:TV} and obtain the desired bound via Jensen: 

\[
    \bbE^k \bsb{ M_T(k) } - \bbE^0 \bsb{ M_T(k) }
\le
    \sum_{t=2}^T \cprob \cdot \e \cdot \sqrt{ 2 \bbE^0 [N_{t-1}(k)] }
\le
    \cprob \cdot \e \cdot T \cdot \sqrt{ 2 \bbE^0 [N_T(k)] }. \qedhere
\]    
\end{proof}
Averaging the quantitative bounds in \Cref{claim:finale} for all $k$ in $[K]$ and applying Jensen's inequality we get the following :
    \begin{align}
    \nonumber
    \frac{1}{K} \sum_{k \in [K]} \bbE^k[M_T(k)]
&\le
    \frac{1}{K} \sum_{k \in [K]} \bbE^0[M_T(k)] + \cprob \cdot \e \cdot T \cdot \sqrt{ \frac{2}{K} \sum_{k \in [K]}\bbE^0 \lsb{  N_T(k) } }
\\
\nonumber
&=
    \frac{1}{K} \bbE^0[M_T] + \cprob \cdot \e \cdot T \cdot \sqrt{ \frac{2}{K} \bbE^0[N_T] }\\
    \label{eq:M_TvsN_T}
&\le
    \lrb{ \frac{1}{10} +  \cprob \cdot \e \cdot \sqrt{ \frac{2}{K} \bbE^0[N_T] } } \cdot T,
\end{align}
    where in the last inequality we used the particular choice of $K=\lceil T^{1/4}\rceil$ and the fact that $T \ge 8008$.

    We have all the ingredients to directly lower bound the average regret suffered by $\cA$. Note that every time an exploring action is played the learner suffers instantaneous regret that is $(\cplat + \cspike \cdot \e)$, while everytime a (wrong) exploiting action is played the learner suffers only $\cspike \cdot \e$:
    \begin{align*}
        \frac{1}{K} \sum_{k\in [K]} \tilde R_T^k (\cA)
    &=
        \frac{1}{K} \sum_{k\in [K]} \Brb{ \cspike \cdot \e \cdot \bbE^k\bsb{ T - M_T(k) - N_T }  + (\cplat + \cspike \cdot \e) \cdot \bbE^k [ N_T ] }
    \\
    &
    \ge
        \cspike \cdot \e \lrb{ T - \frac{1}{K} \sum_{k\in [K]} \bbE^k\bsb{ M_T(k) }} \\
&\ge
    \cspike \cdot \e \cdot  \lrb{ \frac{9}{10} - \cprob \cdot \e \cdot \sqrt{ \frac{2}{K} \bbE^0[N_T] } } \cdot T.
\end{align*}

We have now two cases: if $\cprob \cdot \e \cdot \sqrt{ \frac{2}{K} \bbE^0[N_T] }$ is at most $\tfrac 1{10}$, then the previous inequality yields that
\[
    \frac{1}{K} \sum_{k\in [K]} \tilde R_T^k (\cA)
\ge
    \frac45 \cspike \cdot \e T\ge
    \frac{1}{50^3} T^{3/4} \;.
\]
If, on the other hand, it holds that $\cprob \cdot \e \cdot \sqrt{ \frac{2}{K} \bbE^0[N_T] } > \frac1{10}$, then
\[
    \tilde R_T^0(\cA)
\ge
    \cplat \bbE^0 [N_T] 
>   
    \frac{1}{50^3} T^{3/4} \;. \qedhere
 \]
\end{proof}
\end{document}